%% file: main.tex
\documentclass[10pt,twocolumn,letterpaper]{article}

\usepackage[pagenumbers]{iccv} 

\input{preamble}

\definecolor{iccvblue}{rgb}{0.21,0.49,0.74}
\usepackage[pagebackref,breaklinks,colorlinks,allcolors=iccvblue]{hyperref}

\def\paperID{*****} 
\def\confName{ICCV}
\def\confYear{2025}

\title{\textcolor{GenColor}{Gen}\textcolor{RepColor}{Hancer}: Imperfect Generative Models are Secretly \\ Strong Vision-Centric Enhancers}

\author{Shijie Ma\textsuperscript{1,2}, 
	  ~Yuying Ge\textsuperscript{1,\Letter},
        ~Teng Wang\textsuperscript{1},
        ~Yuxin Guo\textsuperscript{1,2},
        ~Yixiao Ge\textsuperscript{1},
        ~Ying Shan\textsuperscript{1}\\
	$^1$ARC Lab, Tencent PCG\qquad $^2$Institute of Automation, CAS \\
    \url{https://mashijie1028.github.io/GenHancer}
}

\begin{document}

\twocolumn[{%
\renewcommand\twocolumn[1][]{#1}%
\maketitle
\begin{center}
    \centering
    \captionsetup{type=figure}
    \vspace{-7pt}
    \includegraphics[width=.95\textwidth]{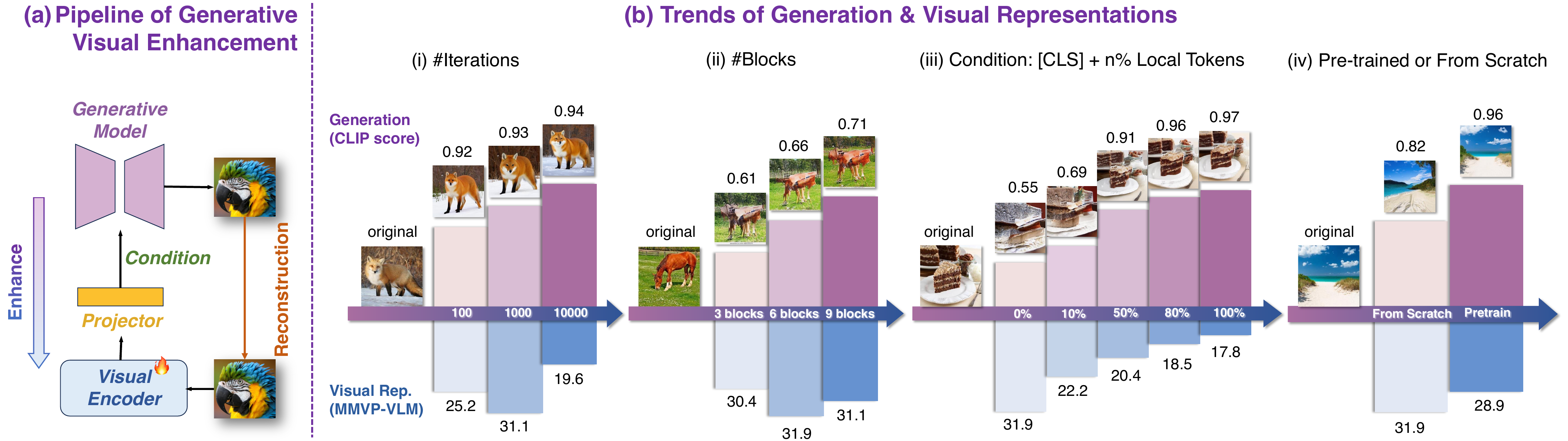}
    \vspace{-7pt}
    \captionof{figure}{\textbf{Perfect \textcolor{GenColor}{generation (reconstruction)} does not always yield desirable \textcolor{RepColor}{visual representations}.} (a) Pipeline of fine-grained visual enhancements, where generative models take visual tokens as conditions and perform reconstruction. (b) Experiments across four dimensions, \ie, training iterations, denoiser size, ratio of local tokens as conditions, and whether to use pre-trained denoisers. We measure \textbf{\textcolor{GenColor}{generation (CLIP score $\uparrow$)}} and \textbf{\textcolor{RepColor}{visual representations (MMVP-VLM $\uparrow$)}} performance. As the results demonstrate, although increasing the number of training iterations, adding more denoiser blocks, using a larger ratio of local tokens as conditions, and employing pre-trained denoisers lead to better generation results, the performance of visual representations does not always improve. Best viewed zoomed in.}
    \label{fig:teaser}
\end{center}
}]

\input{sec/0_abstract}    
\input{sec/1_intro}

\input{sec/2_related_works}

\input{sec/3_preliminaries}

\input{sec/4_method}
\input{sec/5_experiments}

\input{sec/6_conclusive_remarks}

{
    \small
    \bibliographystyle{ieeenat_fullname}
    \bibliography{main}
}

\input{sec/X_appendix}

\end{document}

%% file: preamble.tex
\usepackage{dsfont}
\usepackage{xcolor}
\usepackage{color, colortbl}
\usepackage{booktabs}
\usepackage{multirow}
\usepackage{makecell}
\usepackage{pifont}
\usepackage{amssymb}
\usepackage{subcaption}
\usepackage{amsmath}
\usepackage{fontawesome5}   

\definecolor{color1}{HTML}{ECF4F9}
\definecolor{color2}{HTML}{FFF1E0}
\definecolor{color3}{HTML}{ECF4E9}

\definecolor{GenColor}{HTML}{6D339C}
\definecolor{RepColor}{HTML}{285D9B}

\definecolor{Gray}{gray}{0.9}
\definecolor{LightBlue}{RGB}{236,244,249}

\newcommand{\CC}[1]{\cellcolor{LightBlue}}
\newcommand{\RC}[1]{\rowcolor{LightBlue}}

\usepackage{amsthm}
\newtheorem{theorem}{Theorem}

\usepackage{algorithm}
\usepackage{algpseudocode}
\algrenewcommand{\algorithmicrequire}{\textbf{Input:}}
\algrenewcommand{\algorithmicensure}{\textbf{Output:}}

\usepackage{marvosym}

%% file: sec/0_abstract.tex
\begin{abstract}
The synergy between generative and discriminative models receives growing attention. While discriminative Contrastive Language-Image Pre-Training (CLIP) excels in high-level semantics, it struggles with perceiving fine-grained visual details. Generally, to enhance representations, generative models take CLIP's visual features as conditions for reconstruction. However, the underlying principle remains underexplored. In this work, we empirically found that \textbf{visually} perfect generations are not always optimal for representation enhancement. The essence lies in effectively extracting fine-grained knowledge from generative models while mitigating irrelevant information. To explore critical factors, we delve into three aspects: (1) Conditioning mechanisms: We found that even a small number of local tokens can drastically reduce the difficulty of reconstruction, leading to collapsed training. We thus conclude that utilizing \textbf{only} global visual tokens as conditions is the most effective strategy. (2) Denoising configurations: We observed that end-to-end training introduces extraneous information. To address this, we propose a two-stage training strategy to prioritize learning useful visual knowledge. Additionally, we demonstrate that lightweight denoisers can yield remarkable improvements. (3) Generation paradigms: We explore both continuous and discrete denoisers with desirable outcomes, validating the versatility of our method. Through our in-depth explorations, we have finally arrived at an effective method, namely GenHancer, which consistently outperforms prior arts on the MMVP-VLM benchmark, e.g., 6.0\% on OpenAICLIP. The enhanced CLIP can be further plugged into multimodal large language models for better vision-centric performance.
All the models and codes are made publicly available.
\end{abstract}

%% file: sec/1_intro.tex
\section{Introduction}
\label{sec:intro}

Generative and discriminative models have evolved rapidly in recent years~\cite{liu2023visual,xu2024demystifying,tian2025visual,brooks2024video}. Both of them exhibit complementary strengths, where generative models like diffusion models~\cite{yang2023diffusion,ho2020denoising,rombach2022high} and rectified flow~\cite{lipman2023flow,esser2024scaling} capture low-level visual details, while discriminative models like Contrastive Language-Image Pre-Training (CLIP)~\cite{radford2021learning,zhai2023sigmoid} and DINO~\cite{oquab2024dinov} excel in high-level semantics. This complementary nature enables a synergistic relationship between them. Pioneering work~\cite{yu2025representation} has shown that discriminative models can facilitate the training of generative models through feature alignment. Conversely, generative models can also enhance discriminative models by improving their ability to understand fine-grained visual patterns, \eg, orientation, color and quantity. This enhancement is particularly pertinent for models like CLIP, which have inherent visual shortcomings~\cite{tong2024eyes} that could also limit Multimodal Large Language Models (MLLMs)~\cite{liu2024improved,tong2024cambrian} in vision-centric tasks. Recent works~\cite{hudson2024soda,wang2025diffusion,wang2025reconstructive} have attempted to enhance CLIP ViT by using the visual features of ViT~\cite{dosovitskiy2021an} as conditional inputs for generative models. These models perform self-supervised reconstruction to compel the discriminative model to capture fine-grained visual details, as illustrated in Fig.~\ref{fig:teaser} (a). While these approaches demonstrate the potential of enhancing representations through generative models, they often rely on pre-trained heavy denoisers and do not explore the underlying principle.

\begin{figure}[!t]
    \centering
    \includegraphics[width=.95\linewidth]{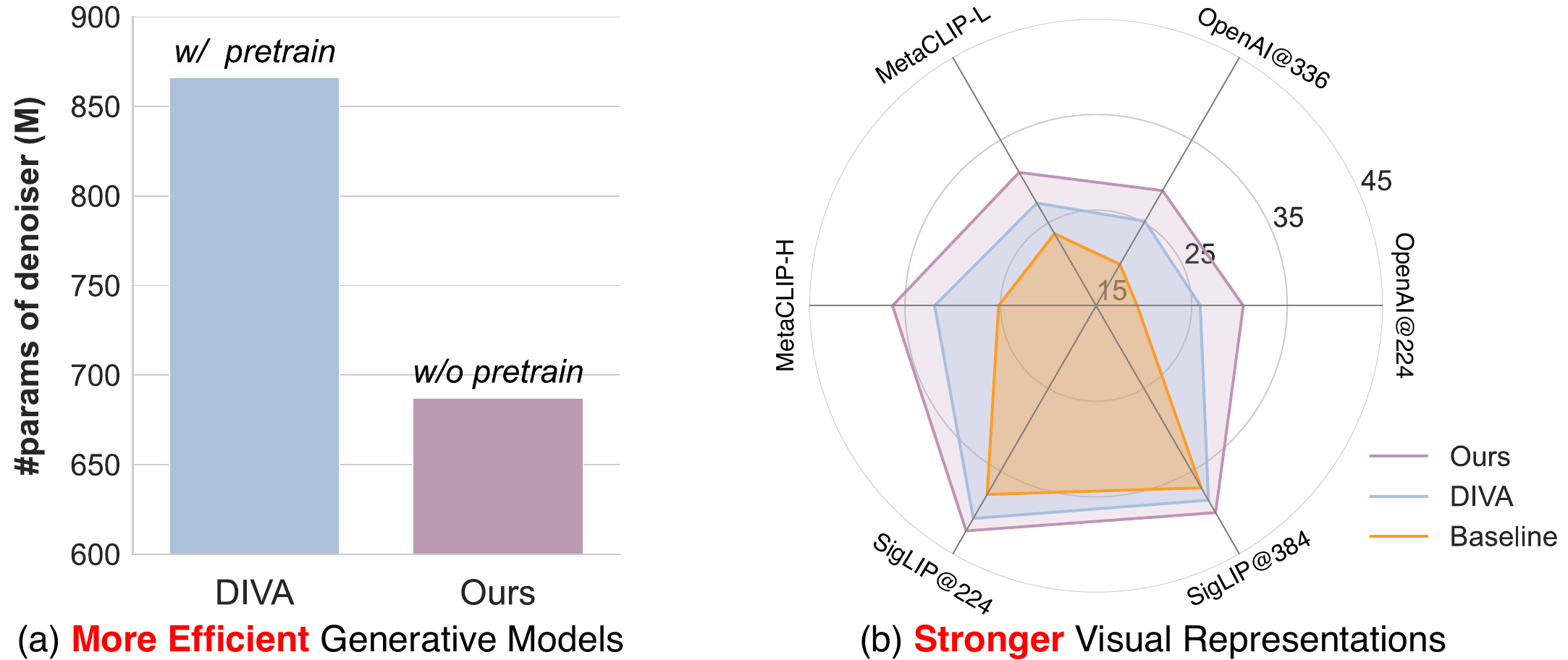}
    \vspace{-7pt}
    \caption{Comparison with prior method~\cite{wang2025diffusion}. (a) We only need a lightweight denoiser, but (b) achieve stronger performance than DIVA~\cite{wang2025diffusion}, which relies on pre-trained heavy generative models.}
    \vspace{-7pt}
    \label{fig:compare-diva}
\end{figure}

To enable generative models to enhance visual representations, a natural question arises: \emph{Do we need a perfect generative model to achieve this enhancement?} To address this question, we conducted preliminary experiments from several dimensions, including \#training iterations, the ratio of local tokens as conditions, the size of denoisers, and whether to use a pre-trained generative model (denoiser), as in Fig.~\ref{fig:teaser} (b). The answer is that perfect generation (reconstruction) does not always yield desirable visual representations. For example, in Fig.~\ref{fig:teaser} (iii), introducing more local tokens as conditions can significantly improve reconstruction, while the visual enhancement will be drastically degraded. In Fig.~\ref{fig:teaser} (iv), although the pre-trained denoiser exhibits better reconstruction, its representations are weaker.

This leads us to further investigate the key points for generative models to effectively enhance visual representations.
We argue that generative models simultaneously contain useful knowledge, like visual patterns and details, as well as irrelevant information, like the gap between CLIP ViT's feature space and generative models' condition space. To effectively enhance representations, our general \emph{philosophy} is that discriminative models should prioritize learning useful knowledge from generative models while circumventing irrelevant information. Furthermore, generative models can be divided into continuous~\cite{lipman2023flow,liu2023flow} and discrete~\cite{esser2021taming} ones, with different denoising objectives, which should also be considered.
Consequently, we conduct in-depth explorations from three key aspects: conditioning mechanisms, denoising configurations, and generation paradigms.

\textbf{Key Point \#1: Which part of the visual information should generative models focus on?}
As in Fig.~\ref{fig:teaser} (a), generative models take visual tokens of discriminative models as conditions. The choice of different tokens significantly impacts the outcomes. In this regard, we find \emph{only} the global token (\ie, class token) could yield desirable visual enhancements. We attribute this to the fact that class token \emph{alone} helps maximize mutual information between visual representations and generative models, while local tokens bring about information leakage and drastically reduce the task’s difficulty, resulting in collapsed learning.

\textbf{Key Point \#2: How to design denoising configurations to transfer useful information for visual representations?}
The structure of the denoiser could determine the enhancement effects. Additionally, before training CLIP, it is essential to mitigate irrelevant information. Therefore, we investigate the influence of different sizes of the denoiser and training stages.
In this paper, we propose GenHancer, a two-stage post-training method for visual enhancements. In the first stage, we pre-train the projector and denoiser while freezing the ViT, learning basic reconstruction abilities and mitigating irrelevant information. In the second stage, we fine-tune CLIP ViT to enhance its fine-grained visual representations. Meanwhile, we empirically found that a lightweight denoiser is sufficient to achieve remarkable results, which is more efficient yet stronger, as in Fig.~\ref{fig:compare-diva}.

\textbf{Key Point \#3: Do two types of denoisers share a common enhancing principle for visual representations?}
For both continuous and discrete denoisers, we present tailor-made designs, including denoiser and conditioning structure. Moreover, we reveal that previous Key Points \#1, \#2 apply to both types, indicating the versatility of our method.

Our contributions are summarized as follows:
\begin{itemize}
    \item We conduct an in-depth study on visual representation enhancements with generative models and make the innovative discovery that perfect reconstruction and pre-trained models are not necessary. This leads us to explore three key aspects: conditioning mechanisms, denoising configurations, and the generation paradigms.
    \item We propose GenHancer, a two-stage post-training method with only lightweight denoisers for visual enhancements, which uses only the class token as the conditional input to perform self-supervised reconstruction. Our method is applicable to both continuous and discrete denoisers.
    \item Comprehensive vision-centric evaluations show that our enhanced CLIP significantly outperforms prior methods that rely on pre-trained heavy denoisers, as in Fig.~\ref{fig:compare-diva}.
\end{itemize}

%% file: sec/2_related_works.tex
\section{Related Works}
\label{sec:related-works}

\paragraph{MLLMs and Vision Encoders.}
Currently, MLLMs~\cite{guo2025aligned,lin2025toklip} predominantly employ CLIP~\cite{radford2021learning,tan2024saco} for visual encoding. Tong \etal~\cite{tong2024eyes} identified several failure patterns in CLIP, which hinder the fine-grained visual understanding. To overcome this issue, early efforts~\cite{tong2024eyes,tong2024cambrian,kar2024brave} employed an ensemble of visual experts to combat the visual shortcomings. More recently, ROSS~\cite{wang2025reconstructive} leverages intrinsic visual activations and incorporates a self-supervised visual reconstruction loss during training MLLMs. Complementarily, DIVA~\cite{wang2025diffusion} proposes to enhance CLIP's fine-grained abilities through diffusion feedback. Similar to~\cite{wang2025diffusion}, we independently enhance CLIP's internal representations, which not only strengthen CLIP as a vision-language retriever but also enable the enhanced CLIP to be seamlessly integrated into MLLMs in a \emph{plug-and-play} manner for better fine-grained vision-centric performance.

\vspace{-13pt}
\paragraph{Enhancing Visual Representations with Diffusion Models.}
Early works~\cite{tian2023stablerep,luo2025deem,shipard2023diversity} utilize generative models as data augmenters~\cite{ma2024active,ma2025towards,tan2025diffin}. Another line of works~\cite{wei2023diffusion,hudson2024soda,chen2025deconstructing} leverages self-supervised reconstruction tasks with diffusion models, which helps models grasp visual details and learn fine-grained representations. Similarly, DIVA~\cite{wang2025diffusion} takes CLIP's features as conditional inputs to the diffusion model~\cite{rombach2022high}, addressing its visual shortcomings through reconstruction. In summary, prior arts predominantly rely on diffusion models~\cite{fuest2024diffusion,rombach2022high}, whereas we apply our method to both continuous and discrete generative models.

\vspace{-13pt}
\paragraph{Vision-Centric Benchmarks.}
Canonical evaluations of MLLMs focus on fundamental multimodal Q\&A capabilities across various domains, \eg, general perception and cognition~\cite{fu2023mme}, text and characters~\cite{singh2019towards}, scientific fields~\cite{lu2022learn}, and potential hallucinations~\cite{li2023evaluating,guan2024hallusionbench} in MLLMs. However, these benchmarks could not effectively assess a model's fine-grained~\cite{guo2024crossmae,guo2025aligned} visual perception abilities, such as object color, quantity, orientation, and viewpoint. To solve this issue, Tong \etal~\cite{tong2024eyes} systematically explore the failure modes of CLIP and propose a challenging MMVP benchmark with 9 visual patterns. CV-Bench~\cite{tong2024cambrian} further expands with 2,600 vision-centric VQA questions, covering dimensions like spatial relationships, count, depth, and distance of both 2D and 3D domains. Besides, NaturalBench~\cite{li2024naturalbench} curates natural adversarial samples that are easy for humans but MLLMs struggle with. In this paper, we employ these vision-centric benchmarks to comprehensively evaluate models' fine-grained visual abilities.

%% file: sec/3_preliminaries.tex
\section{Preliminaries of Generative Models}
\label{sec:preliminaries}

In principle, generative models can be divided into continuous and discrete ones. For continuous generative models, we focus on the recently popular rectified flow~\cite{liu2023flow,lipman2023flow}, while discrete generative models are conventionally built upon pre-trained codebooks~\cite{van2017neural,esser2021taming} for discrete modeling.

\vspace{-12pt}
\paragraph{Rectified Flow (RF).}
Most generative models explicitly or implicitly learn a mapping from a basic distribution, \eg, Gaussian distribution $\mathcal{N}(\boldsymbol{0},\boldsymbol{I})$, to a target distribution, typically the real data distribution $p_\text{data}$. The core idea of RF is to learn an Ordinary Differential Equation (ODE) $\text{d}Z_t=\boldsymbol{u}(Z_t,t)\text{d}t$ that follows a straight path from $\pi_0$ to $\pi_1$. Here $\boldsymbol{u}(Z_t,t)$ is a time-conditional velocity field. This could be achieved by solving a least squares regression problem: $\min_{\boldsymbol{u}}\int_{0}^1\mathbb{E}\big[\Vert (X_1-X_0)-\boldsymbol{u}(X_t,t)\Vert^2\big]\text{d}t$,
where $X_t=tX_1+(1-t)X_0$. In practice, we use $\phi$ to parameterize $\boldsymbol{u}$, and $t$ is basically sampled from the uniform distribution $\mathcal{U}(0,1)$. The learning objective of RF is:
\begin{equation}
\begin{aligned}
    \mathcal{L}_\text{RF}&=\mathbb{E}_{t,\boldsymbol{x}_0,\boldsymbol{x}_1}\Big\Vert (\boldsymbol{x}_1-\boldsymbol{x}_0)-\boldsymbol{u}_\phi\big(t\boldsymbol{x}_1+(1-t)\boldsymbol{x}_0,t\big) \Big\Vert_2^2, \\
    & \text{where}\quad t\sim\mathcal{U}(0,1),\ \boldsymbol{x}_0\sim \mathcal{N}(\boldsymbol{0},\boldsymbol{I}),\ \boldsymbol{x}_1\sim p_\text{data}.
    \label{eq:flow}
\end{aligned}
\vspace{-5pt}
\end{equation}

\vspace{-12pt}
\paragraph{Discrete Generative Models.}
For discrete modeling, one should first learn a discrete codebook, where images are represented by their corresponding indices. For example, VQ-GAN~\cite{esser2021taming} employs some schemes~\cite{goodfellow2014generative,zhang2018unreasonable} to learn a discrete codebook of perceptually rich representations. Subsequently, given indices $s_{<i}$ of image $\boldsymbol{x}$, the discrete generative model $\boldsymbol{p}_\phi$ learns to predict the categorical distribution of the next index $s_i$ via the cross-entropy objective:
\begin{equation}
    \mathcal{L}_\text{CE}=\mathbb{E}_{\boldsymbol{x}\sim p_\text{data}}-\log \prod_{i=1}^L \boldsymbol{p}_\phi(s_i|s_{<i}),
    \label{eq:vq-gan}
\vspace{-7pt}
\end{equation}
where $L$ denotes the sequence length of a sample. $\boldsymbol{p}_\phi$ could be any form of model capable of modeling discrete distributions, \eg, PixelCNNs~\cite{van2016conditional} and Transformers~\cite{vaswani2017attention,esser2021taming}.

\vspace{-12pt}
\paragraph{Conditional Generation.}
To achieve conditional generation, one could incorporate the condition $\boldsymbol{c}$, \eg, class labels or text prompts, into the parameterized model in Eq.~\eqref{eq:flow} and Eq.~\eqref{eq:vq-gan} as $\boldsymbol{u}_\phi(\boldsymbol{x}_t,t,\boldsymbol{c})$ and $\boldsymbol{p}_\phi(s_i|s_{<i},\boldsymbol{c})$, respectively.

%% file: sec/4_method.tex
\section{Method}
\label{sec:method}

\begin{figure*}[!t]
    \centering
    \includegraphics[width=0.95\linewidth]{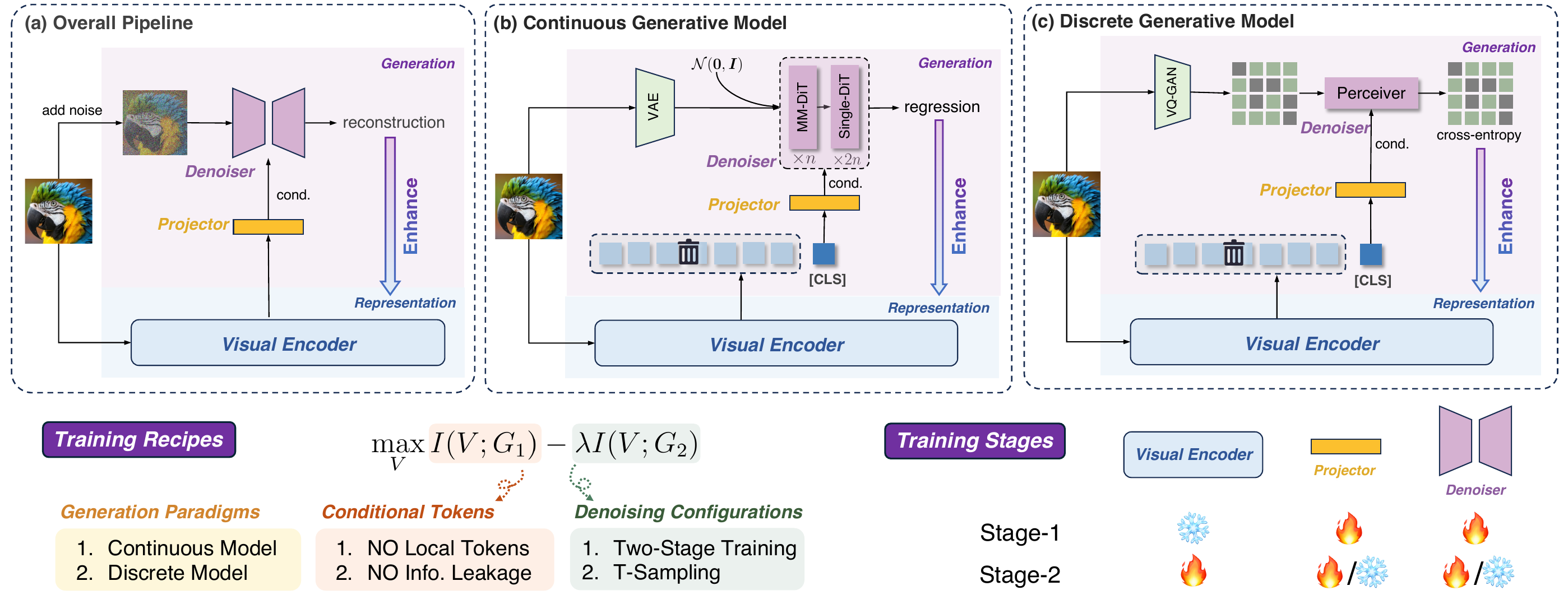}
    \vspace{-7pt}
    \caption{The two-stage post-training framework for visual enhancements. (a) Overall training pipeline. (b) Continuous generative model as the denoiser. We employ a lightweight \texttt{FLUX}-like DiT~\cite{flux2024} (but with fewer blocks) and employ a regression loss of flow matching. (c) Discrete generative model as the denoiser. We choose a lightweight Perceiver~\cite{jaegle2021perceiver} and employ cross-entropy loss to predict masked tokens.}
    \vspace{-5pt}
    \label{fig:method}
\end{figure*}

\subsection{Overview and Formulation}

\paragraph{Overview.}
We propose a two-stage post-training method, namely GenHancer,  to enhance CLIP ViT's fine-grained representations, as in Fig.~\ref{fig:method} (a). To capture key information from generative models, we delve into three aspects:
First, the choice of visual tokens for condition determines the difficulty of the reconstruction task, which is crucial for enhancement (Sec.~\ref{subsec:condition}). Second, we introduce denoising configurations, which enable ViT to capture useful knowledge while mitigating irrelevant information (Sec.~\ref{subsec:train-config}). Third, we present tailored design for both continuous and discrete generative models (Sec.~\ref{subsec:model-arch}), also shown in Fig.~\ref{fig:method} (b), (c).

\vspace{-12pt}
\paragraph{Notations.}
Here, two types of generative models are uniformly represented as $\boldsymbol{g}_\phi$ parameterized by $\phi$. Let $\boldsymbol{v}_\theta$ denote CLIP's visual encoder with parameters $\theta$, whose features are connected to $\boldsymbol{g}_\phi$ as conditions through projector $\boldsymbol{h}_\omega$, \ie, $\boldsymbol{h}_\omega\circ\boldsymbol{v}_\theta(\boldsymbol{x})$. The input sample is $\boldsymbol{x}$, which becomes $\widetilde{\boldsymbol{x}}$ in the denoising space, \eg, VAE~\cite{kingma2013auto} and VQ-GAN~\cite{esser2021taming} for continuous and discrete denoisers, respectively.

\vspace{-12pt}
\paragraph{Repurposing Conditional Generation to Self-supervised Reconstruction.}
Generative models can capture low-level details. To transfer this capability to $\boldsymbol{v}_\theta$, we replace the original condition $\boldsymbol{c}$ with the visual feature $\boldsymbol{v}_\theta(\boldsymbol{x})$.
By reconstructing the visual inputs, $\boldsymbol{v}_\theta$ learns to grasp low-level visual details and is enhanced with fine-grained representations.
In this sense, we transform the original conditional generation into a self-supervised reconstruction task. The learning objectives for continuous $\mathcal{L}_c$ and discrete generative models $\mathcal{L}_d$ can be re-written as Eq.~\eqref{eq:loss-continuous} and Eq.~\eqref{eq:loss-discrete}:
\begin{equation}
\begin{aligned}
    \mathcal{L}_c&=\mathbb{E}_{t,\boldsymbol{x},\widetilde{\boldsymbol{x}}_0,\widetilde{\boldsymbol{x}}_1}\big\Vert (\widetilde{\boldsymbol{x}_1}-\widetilde{\boldsymbol{x}_0})-\boldsymbol{g}_\phi\big(\widetilde{\boldsymbol{x}_t},t,\boldsymbol{h}_\omega\circ\boldsymbol{v}_\theta(\boldsymbol{x})\big) \big\Vert_2^2, \\
    &\text{where}\quad t\sim\mathcal{U}(0,1),\ \widetilde{\boldsymbol{x}_t}=t\widetilde{\boldsymbol{x}_1}+(1-t)\widetilde{\boldsymbol{x}_0},
    \label{eq:loss-continuous}
\end{aligned}
\end{equation}

\vspace{-20pt}
\begin{equation}
    \mathcal{L}_d=\mathbb{E}_{\boldsymbol{x}}-\log \prod_{i=1}^L \boldsymbol{g}_\phi\big(s_i|s_{<i},\boldsymbol{h}_\omega\circ\boldsymbol{v}_\theta(\boldsymbol{x})\big).
    \label{eq:loss-discrete}
\end{equation}
Here, $\boldsymbol{h}_\omega\circ\boldsymbol{v}_\theta(\boldsymbol{x})$ serves as the conditional input of $\boldsymbol{g}_\phi$.

\vspace{-7pt}
\paragraph{Formulation.}
Let $G$ and $V$ denote random variables of features of $\boldsymbol{g}_\phi$ and $\boldsymbol{v}_\theta$. $I(\cdot)$ and $H(\cdot)$ denote mutual information and entropy. Then we have the following theorem:
\begin{theorem}
\label{theorem:entropy}
   When $\boldsymbol{g}_\phi$ is fixed, self-supervised reconstruction is equivalent to maximizing the mutual information $I(V;G)$ between $V$ and $G$. The knowledge learned by $\boldsymbol{v}_\theta$ from $\boldsymbol{g}_\phi$ can be interpreted as the increase in $I(V;G)$.
\end{theorem}
\begin{proof}
    The mutual information could be written as: $I(V;G)=H(G)-H(G|V)$. Through reconstruction in Eq.~\eqref{eq:loss-continuous} or Eq.~\eqref{eq:loss-discrete}, by conditioning $G$ on $V$, $V$ is trained to approximate the distribution of $G$. Consequently, $H(G|V)$ decreases during training. While $H(G)$ is fixed, the decrease in $H(G|V)$ leads to the increase in $I(V;G)$.
    \vspace{-5pt}
\end{proof}
From the results in Fig.~\ref{fig:teaser} (b)(i), the reconstruction improves as training progresses, which corresponds to an increasing $I(V;G)$. However, visual representations might decrease. In light of this, for the enhancement of visual representations, the knowledge in $G$ can be decomposed into useful knowledge $G_1$ (\eg, basic semantics, visual patterns) and irrelevant information $G_2$ like the gap between feature space of $\boldsymbol{v}_\theta$ and condition space of $\boldsymbol{g}_\phi$. In this regard, to effectively enhance visual representations, \textbf{our underlying \emph{philosophy} is: The visual encoder should learn to capture useful knowledge from generative models as much as possible, \ie, $\max I(V;G_1)$, while avoiding irrelevant information, \ie, $\min I(V;G_2)$.} This equals to applying regularization on $V$ to prevent overfitting to $G_2$:

\begin{equation}
    \max_V I(V;G_1)-\lambda I(V;G_2) \Rightarrow \max_V I(V;G_1)+\lambda d(V;V_0),
    \label{eq:philosophy}
\end{equation}
$V_0$ is the initial visual model and $d(\cdot)$ is a distance metric.

\subsection{Conditional Visual Tokens}
\label{subsec:condition}

The choice of conditional visual tokens is crucial for visual enhancement. If too many tokens are fed to the generative model, the reconstruction becomes excessively easy. The reason is that local tokens directly correspond to image areas with information leakage. In this case, $I(V;G_1)$ in Eq.~\eqref{eq:philosophy} becomes small and $\boldsymbol{v}_\theta$ fails to grasp useful information from $\boldsymbol{g}_\phi$. To ensure a remarkable $I(V;G_1)$, we argue that the number of local tokens should be carefully controlled. Our experiments show that even a small number of local tokens, though achieving good reconstruction quality, can still cause marginal visual enhancement, as in Fig.~\ref{fig:teaser} (iii).
As a result, we propose that the visual condition features should exclusively comprise \emph{only} the class token \texttt{[CLS]}. This strategy applies to both continuous and discrete models, as validated in Fig.~\ref{fig:ratio-local} of Sec.~\ref{subsec:ablations}.

\subsection{Denoising Configurations}
\label{subsec:train-config}

To effectively enhance visual representations, we aim to maximize $I(V;G_1)$ while suppressing $I(V;G_2)$ in Eq.~\eqref{eq:philosophy}. In this regard, our explorations are three-fold: training stages, timestamp sampling of the continuous denoiser, and the update strategy for $\boldsymbol{v}_\theta$.

\vspace{-12pt}
\paragraph{Two Stage Training.}
An important source of $G_2$ is the gap between the feature space of $\boldsymbol{v}_\theta$ and the conditions of $\boldsymbol{g}_\phi$, which is irrelevant to representation learning and could degrade the performance. Furthermore, since $\boldsymbol{g}_\phi$ is lightweight and randomly initialized, it could introduce potential noise to $\boldsymbol{v}_\theta$ at the beginning. Consequently, we propose a two-stage training pipeline. At Stage-1, we train the denoiser $\boldsymbol{g}_\phi$ and the projector $\boldsymbol{h}_\omega$ while freezing $\boldsymbol{v}_\theta$, in which $\boldsymbol{g}_\phi$ acquires basic generative capabilities for visual enhancements and $\boldsymbol{h}_\omega$ learns to bridge the space gap, thereby reducing $I(V;G_2)$. In Stage-2, we focus on enlarging $I(V;G_1)$ and train $\boldsymbol{v}_\theta$ to improve fine-grained representations. Moreover, we empirically found that as long as Stage-1 is performed sufficiently, the impact of whether the denoiser and projector are trained in Stage-2 is negligible.

\vspace{-12pt}
\paragraph{Low Rank Adaption (LoRA) of $\boldsymbol{v}_\theta$.}
The pre-trained visual encoder $\boldsymbol{v}_\theta$ possesses strong global semantics, \ie, $V_0$, which should be maintained when incorporating fine-grained perception. To prevent $\boldsymbol{v}_\theta$ from overfitting during reconstruction, we update $\boldsymbol{v}_\theta$ using LoRA~\cite{hu2022lora}, which implicitly constrains $d(V,V_0)$ in Eq.~\eqref{eq:philosophy}.

\vspace{-12pt}
\paragraph{Timestamp Sampling.}
For continuous models like RF, timestamp sampling is of vital importance. Conventionally, RF~\cite{liu2023flow} is trained to predict velocity across timestamps uniformly in $[0,1]$. Considering $\boldsymbol{x}_t=t\boldsymbol{x}_1+(1-t)\boldsymbol{x}_0$, prior works~\cite{esser2024scaling} uncover that the velocity target at intermediate timestamps, \ie, $t\approx 0.5$, is more challenging. In our case, sampling intermediate timestamps more frequently could increase the difficulty of the reconstruction task, thus effectively amplifying $I(V;G_1)$ and allowing the visual encoder $\boldsymbol{v}_\theta$ to effectively acquire useful fine-grained knowledge from $G_1$.
In this regard, we propose \emph{scaled} Logit-Normal sampling for timestamps, as shown below:
\begin{equation}
    t=\texttt{sigmoid}(s\cdot\varepsilon),\quad \text{where}\ \ \varepsilon\sim\mathcal{N}(0,1).
    \label{eq:t-sampling}
\end{equation}
Here, $\varepsilon$ is sampled from the normal distribution, $\texttt{sigmoid}(x)=\frac{1}{1+\exp(-x)}$, and $s>0$ is the scale hyperparameter that controls the extent to which sampling is focused on the intermediate timestamps. Smaller $s$ results in more frequent sampling around 0.5. The diagrams of distributions in various $s$ are illustrated in the Appendix.

\subsection{Generation Paradigms}
\label{subsec:model-arch}

For both types of generative models, we need to design architectures for denoisers and implementation of the conditioning mechanism. Notably, our denoiser is lightweight and randomly initialized, without pre-trained weights of heavy denoisers like Stable Diffusion~\cite{rombach2022high} in~\cite{wang2025diffusion}.

\vspace{-12pt}
\paragraph{Continuous Generative Models.}
We choose RF as the continuous denoiser, which is modeled in the latent space of pre-trained VAE~\cite{kingma2013auto}. The structure is inherited from \texttt{FLUX.1-dev}~\cite{flux2024}, consisting of $n\times$ Multimodal Diffusion Transformer (MM-DiT)~\cite{peebles2023scalable,esser2024scaling} blocks and $2n\times$ single-stream DiT (Single-DiT) blocks, as shown in Fig.~\ref{fig:method} (b). By default, we set $n=2$, which is very efficient with $\sim 1/10$ parameters of the original \texttt{FLUX.1-dev} denoiser. Similar to DiT~\cite{peebles2023scalable,esser2024scaling}, the condition of visual tokens (\texttt{[CLS]} of $\boldsymbol{v}_\theta$) is introduced through the modulation mechanism via adaptive layernorm~\cite{peebles2023scalable,ba2016layer}. The learning objective is the regression of flow matching in Eq.~\eqref{eq:loss-continuous}.

\vspace{-12pt}
\paragraph{Discrete Generative Models.}
Here, we choose Perceiver~\cite{jaegle2021perceiver} as the discrete denoiser, building upon off-the-shelf VQ-GAN's codebook~\cite{esser2021taming}. We first mask a certain proportion of input tokens. The condition of visual features is introduced via a cross-attention module, as depicted in Fig.~\ref{fig:method} (c). Specifically, we set the query as the unmasked tokens $s_{<i}$, while the key and value are the concatenation of the unmasked tokens and \texttt{[CLS]} of $\boldsymbol{v}_\theta$. They are collectively fed to the Perceiver with cross-entropy loss to predict the masked token indices $s_i$, as in Eq.~\eqref{eq:loss-discrete}.

%% file: sec/5_experiments.tex
\section{Experiments}
\label{sec:experiments}

\subsection{Experimental Setup}

\paragraph{Implementation Details.}
For continuous generative models, we choose RF, whose structure is similar to \texttt{FLUX.1-dev}~\cite{flux2024}, but with only 2 MM-DiT and 4 Single-DiT blocks ($\sim 10\%$ of the parameters). The discrete denoiser is parameterized by a 6-layer Perceiver to predict the masked tokens indexed VQ-GAN's codebook~\cite{esser2021taming}. Similar to~\cite{xie2025showo}, the mask ratio is randomly sampled from 50\% to 90\%. For both generative models, we only take the \texttt{[CLS]} token of CLIP ViT as the conditional input while dropping other local tokens to prevent information leakage. We choose the scale factor in Eq.~\eqref{eq:t-sampling} as 1 by default.

\begin{table*}[!t]
\setlength\tabcolsep{4pt}
\centering
\renewcommand{\arraystretch}{0.9}
\caption{Performance of various CLIP backbones in MMVP-VLM benchmark. Here, we report our results using the continuous denoiser. The enhanced CLIP consistently outperforms prior methods across various visual patterns. The visual patterns are symbolized as: \textbf{\faCompass}: Orientation and Direction, \textbf{\faSearch}: Presence of Specific Features, \textbf{\faSync}: State and Condition, \textbf{\faSortNumericUp}: Quantity and Count, \textbf{\faMapPin}: Positional and Relational Context, \textbf{\faPalette}: Color and Appearance, \textbf{\faCogs}: Structural and Physical Characteristics, \textbf{\faFont}: Texts, \textbf{\faCamera}: Viewpoint and Perspective.}
\vspace{-7pt}
\label{tab:main-mmvp-clip}
\resizebox{.9\linewidth}{!}{
\begin{tabular}{@{}lccl|ccccccccc|c@{}}
\toprule
CLIP Backbone & \#Params (M) & Resolution & Method & \faCompass & \faSearch & \faSync & \faSortNumericUp & \faMapPin & \faPalette & \faCogs & \faFont & \faCamera & Average \\ \midrule
\multirow{3}{*}{OpenAI ViT-L-14} & \multirow{3}{*}{427.6} & \multirow{3}{*}{224$^2$} & Original & 13.3 & 13.3 & 20.0 & 20.0 & 13.3 & 53.3 & 20.0 & 6.7 & 13.3 & 19.3 \\
 &  &  & + DIVA & 13.3 & 20.0 & 40.0 & 6.7 & 20.0 & 53.3 & 46.7 & 20.0 & 13.3 & 25.9 \\
 &  &  & \CC{30}+ Ours & \CC{30}13.3 & \CC{30}33.3 & \CC{30}33.3 & \CC{30}20.0 & \CC{30}6.7 & \CC{30}73.3 & \CC{30}46.7 & \CC{30}20.0 & \CC{30}40.0 & \CC{30}\textbf{31.9} \textbf{\small{\color{Green}(+6.0)}} \\ \midrule
\multirow{3}{*}{OpenAI ViT-L-14} & \multirow{3}{*}{427.9} & \multirow{3}{*}{336$^2$} & Original & 0.0 & 20.0 & 40.0 & 20.0 & 6.7 & 20.0 & 33.3 & 6.7 & 33.3 & 20.0 \\
 &  &  & + DIVA & 26.7 & 20.0 & 33.3 & 13.3 & 13.3 & 46.7 & 26.7 & 6.7 & 40.0 & 25.2 \\
 &  &  & \CC{30}+ Ours & \CC{30}6.7 & \CC{30}20.0 & \CC{30}33.3 & \CC{30}20.0 & \CC{30}6.7 & \CC{30}73.3 & \CC{30}53.3 & \CC{30}26.7 & \CC{30}26.7 & \CC{30}\textbf{29.6} \textbf{\small{\color{Green}(+4.4)}} \\ \midrule \midrule
\multirow{3}{*}{MetaCLIP ViT-L-14} & \multirow{3}{*}{427.6} & \multirow{3}{*}{224$^2$} & Original & 13.3 & 6.7 & 66.7 & 6.7 & 33.3 & 46.7 & 20.0 & 6.7 & 13.3 & 23.7 \\
 &  &  & + DIVA & 6.7 & 6.7 & 60.0 & 0.0 & 26.7 & 66.7 & 20.0 & 20.0 & 40.0 & 27.4 \\
 &  &  & \CC{30}+ Ours & \CC{30}13.3 & \CC{30}20.0 & \CC{30}53.3 & \CC{30}13.3 & \CC{30}26.7 & \CC{30}80.0 & \CC{30}33.3 & \CC{30}13.3 & \CC{30}33.3 & \CC{30}\textbf{31.9} \textbf{\small{\color{Green}(+4.5)}} \\ \midrule
\multirow{3}{*}{MetaCLIP ViT-H-14} & \multirow{3}{*}{986.1} & \multirow{3}{*}{224$^2$} & Original & 6.7 & 13.3 & 60.0 & 13.3 & 6.7 & 53.3 & 26.7 & 13.3 & 33.3 & 25.2 \\
 &  &  & + DIVA & 13.3 & 20.0 & 53.3 & 33.3 & 13.3 & 66.7 & 33.3 & 13.3 & 40.0 & 31.9 \\
 &  &  & \CC{30}+ Ours & \CC{30}20.0 & \CC{30}20.0 & \CC{30}66.7 & \CC{30}26.7 & \CC{30}26.7 & \CC{30}66.7 & \CC{30}33.3 & \CC{30}20.0 & \CC{30}53.3 & \CC{30}\textbf{37.0} \textbf{\small{\color{Green}(+5.1)}} \\ \midrule \midrule
\multirow{3}{*}{SigLIP ViT-SO-14} & \multirow{3}{*}{877.4} & \multirow{3}{*}{224$^2$} & Original & 26.7 & 20.0 & 53.3 & 40.0 & 20.0 & 66.7 & 40.0 & 20.0 & 53.3 & 37.8 \\
 &  &  & + DIVA & 13.3 & 26.7 & 60.0 & 46.7 & 13.3 & 73.3 & 53.3 & 26.7 & 53.3 & 40.7 \\
 &  &  & \CC{30}+ Ours & \CC{30}20.0 & \CC{30}20.0 & \CC{30}66.7 & \CC{30}60.0 & \CC{30}20.0 & \CC{30}86.7 & \CC{30}40.0 & \CC{30}13.0 & \CC{30}53.3 & \CC{30}\textbf{42.2} \textbf{\small{\color{Green}(+1.5)}} \\ \midrule
\multirow{3}{*}{SigLIP ViT-SO-14} & \multirow{3}{*}{878.0} & \multirow{3}{*}{384$^2$} & Original & 20.0 & 26.7 & 60.0 & 33.3 & 13.3 & 66.7 & 33.3 & 26.7 & 53.3 & 37.0 \\
 &  &  & + DIVA & 26.7 & 33.3 & 53.3 & 26.7 & 13.3 & 80.0 & 40.0 & 26.7 & 46.7 & 38.5 \\
 &  &  & \CC{30}+ Ours & \CC{30}26.7 & \CC{30}20.0 & \CC{30}66.7 & \CC{30}33.3 & \CC{30}13.3 & \CC{30}86.7 & \CC{30}40.0 & \CC{30}26.7 & \CC{30}46.7 & \CC{30}\textbf{40.0} \textbf{\small{\color{Green}(+1.5)}} \\ \bottomrule
\end{tabular}
}
\vspace{-7pt}
\end{table*}

\vspace{-13pt}
\paragraph{Training Details.}
Our training process consists of two stages, each involving one epoch on the CC3M~\cite{sharma2018conceptual} dataset. We choose AdamW as the optimizer, with a learning rate of 1e-4 and 1e-5 for Stage-1 and Stage-2, respectively. At Stage-2, we optimize the visual encoder using LoRA with a rank of 16. We employ a global batch size of 256.

\vspace{-13pt}
\paragraph{Comparative Baseline.}
Similar to~\cite{wang2025diffusion}, our method GenHancer independently enhances CLIP via post-tuning. When equipped with our enhanced CLIP and trained with original recipes, MLLMs could perform better on vision-centric benchmarks. In this regard, GenHancer could be viewed as a \emph{plug-and-play} vision-enhancement method for MLLMs. We primarily compare with DIVA~\cite{wang2025diffusion}.

\vspace{-13pt}
\paragraph{Evaluation Protocol.}
Following~\cite{wang2025diffusion}, we perform visual enhancements on six CLIP backbones, including OpenAICLIP ViT-L @224/@336~\cite{radford2021learning}, MetaCLIP@224 ViT-L/H~\cite{xu2024demystifying} and SigLIP-SO-14 @224/@384~\cite{zhai2023sigmoid}. We use MMVP-VLM~\cite{tong2024eyes} to evaluate fine-grained perception abilities. Subsequently, we follow the official training recipes of LLaVA-1.5~\cite{liu2024improved} to train MLLMs with our enhanced CLIP ViT. The resulting MLLMs are comprehensively evaluated on vision-centric benchmarks like MMVP-MLLM~\cite{tong2024eyes}, CV-Bench~\cite{tong2024cambrian} and NaturalBench~\cite{li2024naturalbench}, as well as multimodal understanding benchmarks, including POPE~\cite{li2023evaluating} ScienceQA~\cite{lu2022learn} and HallusionBench~\cite{guan2024hallusionbench}.

\subsection{Comparative Results}

\paragraph{Our method significantly enhances CLIP's fine-grained visual perception abilities.}
We evaluate CLIP models on the challenging MMVP-VLM benchmark~\cite{tong2024eyes}, which contains 9 fine-grained visual patterns for a comprehensive vision-centric evaluation. As in Table~\ref{tab:main-mmvp-clip}, our method with only a lightweight denoiser, surpasses the previous method~\cite{wang2025diffusion} that employed a heavy pre-trained denoiser across multiple CLIP backbones, with variations in resolution and parameters. For example, our method outperforms DIVA by 6.0\% and 4.5\% on OpenAICLIP and MetaCLIP, respectively.
Besides, CLIP's visual shortcomings are effectively addressed after post-training, \eg, we improved MetaCLIP's color perception (\faPalette) from 46.7\% to 80.0\%, and enhanced its viewpoint understanding (\faCamera) by 20\%.

\begin{figure}[!t]
    \centering
    \includegraphics[width=.95\linewidth]{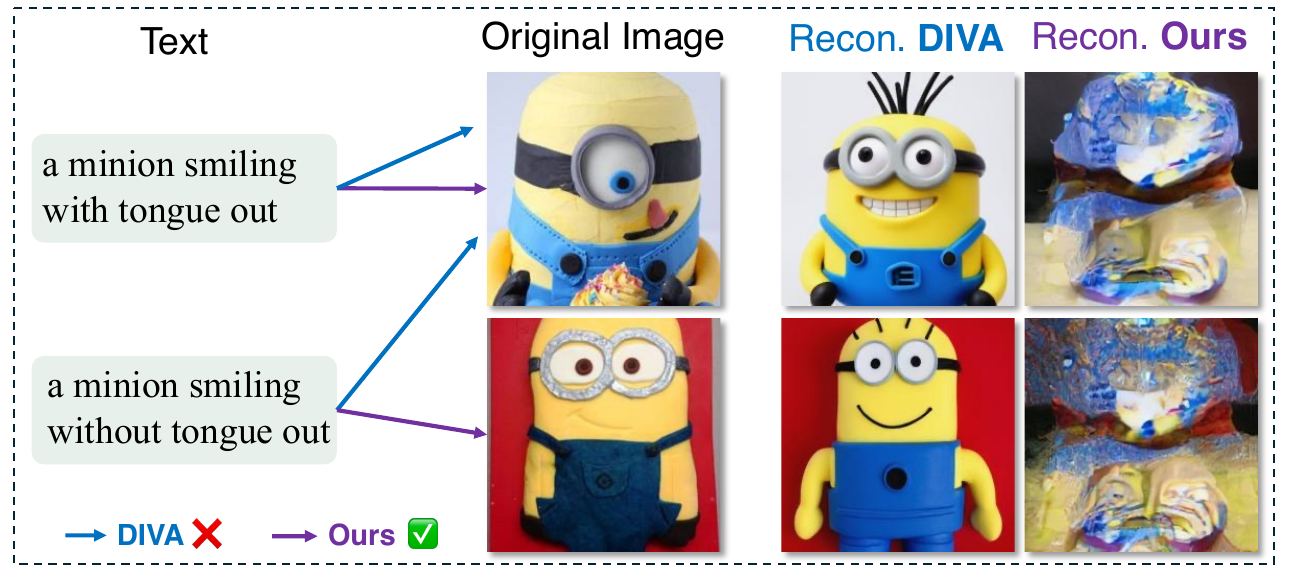}
    \vspace{-7pt}
    \caption{Qualitative results. Although DIVA achieves better reconstructions of input images, it fails to perceive fine-grained visual details between `tongue out' and `without tongue out'.}
    \vspace{-10pt}
    \label{fig:qualitative}
\end{figure}

\vspace{-12pt}
\paragraph{Qualitative Evaluations.}
We present two cases in Fig.~\ref{fig:qualitative}. Although DIVA achieves better reconstructions, our method correctly retrieves images for given texts, while DIVA fails. This further emphasizes that better reconstruction does not necessarily lead to better representations.

\begin{table*}[!t]
\setlength\tabcolsep{3pt}
\centering
\renewcommand{\arraystretch}{0.9}
\caption{Comprehensive evaluation of MLLMs (LLaVA-1.5~\cite{liu2024improved}), including vision-centric and conventional MLLM benchmarks. $^\dagger$ We use official DIVA CLIP checkpoints~\cite{wang2025diffusion} to reproduce the results. $^\ddagger$ Similar to~\cite{li2024seed}, we select the choice with the highest likelihood as MLLM's prediction. Hallusion: HallusionBench~\cite{guan2024hallusionbench}. SciQA: ScienceQA~\cite{lu2022learn}. \textbf{Bold} and \underline{underline} indicate the best and the second best.}
\vspace{-10pt}
\label{tab:mllm-benchmark}
\resizebox{.95\linewidth}{!}{
\begin{tabular}{@{}llccccccccccccc@{}}
\toprule
\multirow{3}{*}{LLM} & \multirow{3}{*}{CLIP} & \multicolumn{8}{c}{Vision-Centric Benchmarks} & \multicolumn{5}{c}{Conventional MLLM Benchmarks} \\ \cmidrule(l){3-10} \cmidrule(l){11-15} 
 &  & \multirow{2}{*}{\begin{tabular}[c]{@{}c@{}}MMVP-\\ MLLM~\cite{tong2024eyes}\end{tabular}} & \multicolumn{4}{c}{NaturalBench~\cite{li2024naturalbench}$^\ddagger$} & \multicolumn{2}{c}{CV-Bench 2D~\cite{tong2024cambrian}} & \multirow{2}{*}{\begin{tabular}[c]{@{}c@{}}CV-Bench\\ 3D~\cite{tong2024cambrian}\end{tabular}} & \multicolumn{3}{c}{POPE~\cite{li2023evaluating}} & \multirow{2}{*}{\begin{tabular}[c]{@{}c@{}}SciQA-\\ IMG~\cite{lu2022learn}\end{tabular}} & \multirow{2}{*}{\begin{tabular}[c]{@{}c@{}}Hallusion\\ Avg.~\cite{guan2024hallusionbench}\end{tabular}} \\ \cmidrule(lr){4-7} \cmidrule(lr){8-9} \cmidrule(lr){11-13}
 &  &  & Acc & Q-Acc & I-Acc & G-Acc & ADE20K & COCO &  & rand & pop & adv &  &  \\ \midrule
\multirow{3}{*}{Vicuna-7B~~~} & Original & 24.7 & \underline{76.4} & \underline{53.6} & \underline{56.4} & 17.6 & 49.6 & 60.9 & 58.7 & 87.3 & 86.1 & 84.2 & \textbf{66.8} & 27.6 \\
 & DIVA$^\dagger$ & \textbf{31.3} & 75.3 & 51.7 & 56.1 & \underline{22.3} & \underline{51.3} & \underline{63.4} & \underline{60.2} & \underline{87.9} & \textbf{87.0} & \textbf{84.6} & 66.3 & \textbf{28.6} \\
 & \CC{30}Ours & \CC{30}\underline{30.7} & \CC{30}\textbf{77.3} & \CC{30}\textbf{55.6} & \CC{30}\textbf{59.1} & \CC{30}\textbf{24.4} & \CC{30}\textbf{52.9} & \CC{30}\textbf{63.6} & \CC{30}\textbf{63.2} & \CC{30}\textbf{88.1} & \CC{30}\underline{86.7} & \CC{30}\textbf{84.6} & \CC{30}\underline{66.5} & \CC{30}\underline{28.4} \\ \midrule \midrule
\multirow{3}{*}{Vicuna-13B~~~} & Original & 30.7 & \underline{76.3} & \underline{52.9} & 55.1 & 13.8 & 52.6 & 63.3 & 65.0 & 87.1 & 86.2 & 84.5 & 71.6 & 24.5 \\
 & DIVA$^\dagger$ & \underline{35.3} & 76.0 & 52.7 & \underline{56.0} & \underline{16.8} & \underline{53.2} & \textbf{64.3} & \underline{65.8} & \textbf{88.1} & \textbf{87.4} & \underline{84.8} & \underline{71.8} & \underline{25.2} \\
 & \CC{30}Ours & \CC{30}\textbf{36.7} & \CC{30}\textbf{77.2} & \CC{30}\textbf{55.3} & \CC{30}\textbf{58.7} & \CC{30}\textbf{22.9} & \CC{30}\textbf{55.3} & \CC{30}\textbf{64.3} & \CC{30}\textbf{66.4} & \CC{30}\underline{87.8} & \CC{30}\underline{87.0} & \CC{30}\textbf{84.9} & \CC{30}\textbf{72.3} & \CC{30}\textbf{26.4} \\ \bottomrule
\end{tabular}
}
\vspace{-7pt}
\end{table*}

\vspace{-12pt}
\paragraph{\emph{Plug-and-play} vision-centric enhancements for MLLMs.}
Our method \emph{independently} enhances CLIP ViT with fine-grained representations. Considering that existing MLLMs~\cite{liu2023visual,liu2024improved,bai2023qwen} predominantly use CLIP ViT as the visual encoder, we replace the original CLIP with the enhanced CLIP as a \emph{plug-and-play} module and integrate it into MLLMs to explore the impact of the enhanced visual representations on MLLMs' final performance. For fair comparisons, we adopt the same training setup as LLaVA-1.5~\cite{liu2024improved}, \ie, training data and stages, to train MLLMs. For DIVA~\cite{wang2025diffusion}, we adopt the official CLIP checkpoints. We conduct a comprehensive evaluation of the MLLMs on multiple vision-centric benchmarks, including  MMVP-MLLM~\cite{tong2024eyes}, CV-Bench~\cite{tong2024cambrian} and NaturalBench~\cite{li2024naturalbench}, as well as some general multimodal understanding benchmarks. Results in Table~\ref{tab:mllm-benchmark} show that visual enhancement of CLIP is effectively transferred to MLLMs, resulting in significant improvements across vision-centric benchmarks. For instance, compared to the original CLIP in Vicuna-7B MLLM, we achieved 6.0\% and 4.5\% improvements on MMVP-MLLM and CV-Bench 3D, respectively.

\begin{table}[!t]
\setlength\tabcolsep{2pt}
\centering
\renewcommand{\arraystretch}{1}
\caption{Performance of zero-shot classification and retrieval that require global semantics. We report the results of original and post-tuned OpenAICLIP@224.}
\vspace{-10pt}
\label{tab:zero-shot-clip}
\resizebox{.98\linewidth}{!}{
\begin{tabular}{@{}lcccccccc@{}}
\toprule
\multirow{2}{*}{Method} & \multicolumn{4}{c}{Classification} & \multicolumn{2}{c}{Retrieval-Image@5} & \multicolumn{2}{c}{Retrieval-Text@5} \\ \cmidrule(l){2-5} \cmidrule(l){6-7} \cmidrule(l){8-9} 
 & IN-1K & C100 & SUN397 & Cars & Flickr30k & COCO & Flickr30k & COCO \\ \midrule
Original & 75.5 & 76.1 & 67.5 & 77.7 & 87.2 & 61.1 & 97.4 & 79.2 \\
\RC{30}Ours & 75.6 & 76.1 & 67.5 & 77.6 & 87.3 & 61.2 & 97.2 & 79.4 \\ \bottomrule
\end{tabular}
}
\vspace{-7pt}
\end{table}

\vspace{-12pt}
\paragraph{Visual enhancements do not hurt CLIP's original global semantics.}
CLIP has inherently strong global semantics in classification-based tasks~\cite{ma2024happy,ma2025protogcd}. To explore how fine-grained enhancements affect this ability, we evaluate zero-shot classification on datasets like ImageNet-1K~\cite{deng2009imagenet}, CIFAR100~\cite{krizhevsky2009learning}, Stanford Cars~\cite{krause20133d}, and SUN397~\cite{xiao2010sun} and zero-shot cross-modal retrieval tasks on Flick30k~\cite{young2014image} and COCO~\cite{chen2015microsoft}.
Table~\ref{tab:zero-shot-clip} reveals that the performance difference is minimal ($<0.3\%$) across various settings, which means that our method could enhance CLIP's fine-grained understanding without forgetting its global semantics~\cite{tan2023datapruneinfomax,tan2023movingonesampleout}.

\begin{figure}[!t]
    \centering
    \includegraphics[width=0.75\linewidth]{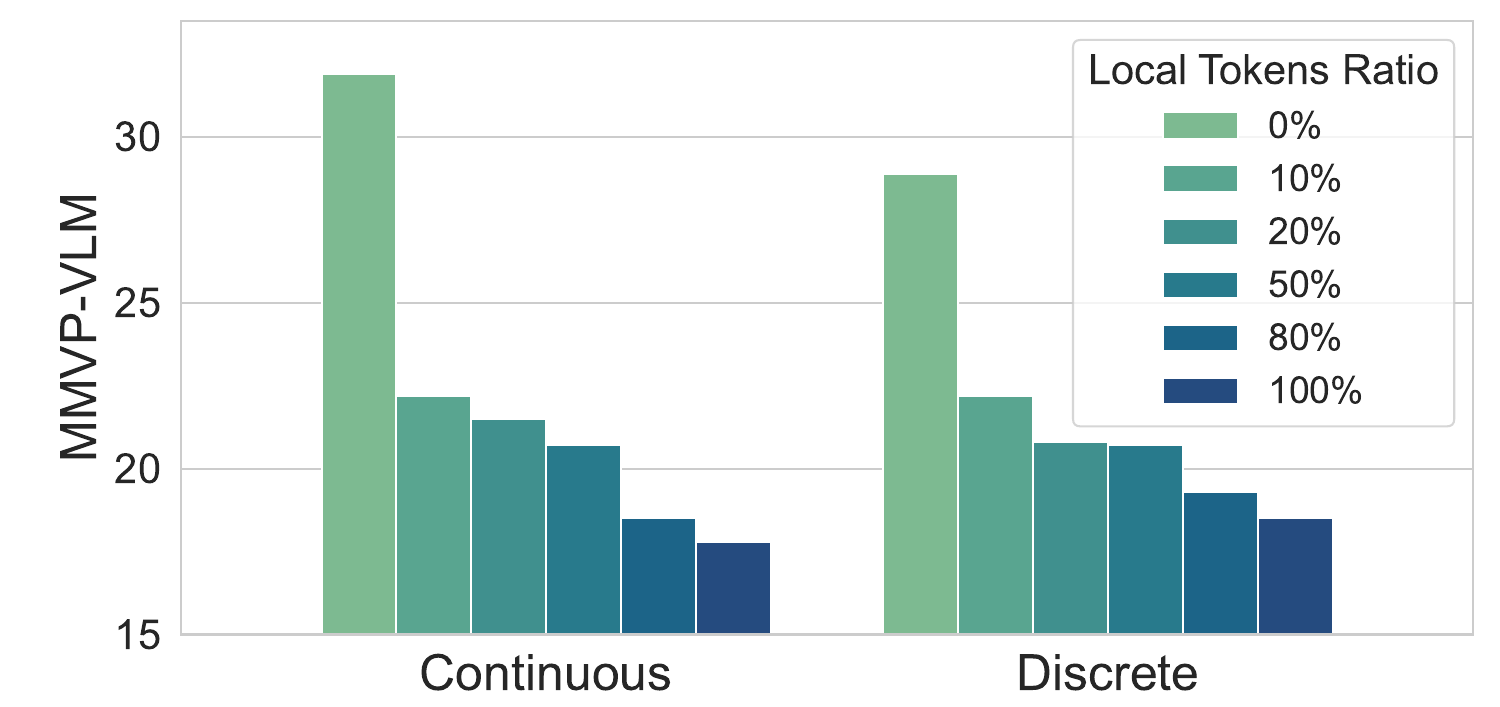}
    \vspace{-7pt}
    \caption{Performance of CLIP across various conditional visual tokens on MMVP-VLM, \ie, \texttt{[CLS]} + $n\%$ \texttt{[LOCAL]}.}
    \vspace{-7pt}
    \label{fig:ratio-local}
\end{figure}

\subsection{Key Explorations and Ablations}
\label{subsec:ablations}

\paragraph{Key Point \#1: Selecting Conditional Visual Tokens.}
As in Sec.~\ref{subsec:condition}, selecting conditional visual tokens is critical for enhancing representations. We conduct experiments by choosing the class token and different proportions of local tokens, \ie, \texttt{[CLS]} + $n\%$ \texttt{[LOCAL]}. As displayed in Fig.~\ref{fig:ratio-local}, even a very small ratio (10\%) leads to significant performance degradation, which suggests that local tokens carry substantial signals for reconstruction, making the task too easy with information leakage. Consequently, this prevents the visual encoder from effectively learning fine-grained details and brings about a limited $I(V; G_1)$.
The conclusion applies to both types of $\boldsymbol{g}_\phi$. Therefore, we propose to choose only the class token as the condition.

\vspace{-12pt}
\paragraph{Key Point \#2.1: Two-Stage Training.}
As elaborated in Sec.~\ref{subsec:train-config}, in Stage-1 of the two-stage training scheme, the projector learns to bridge the gap between the feature space of the visual encoder and the condition space of the denoiser, which serves as irrelevant information $G_2$. Ablations comparing end-to-end with the proposed two-stage training are illustrated in Fig.~\ref{fig:training-stages}. End-to-end training consistently exhibits a performance drop of over 5\% across various settings. This indicates that our two-stage training is crucial in preventing interference from $G_2$.

\begin{figure}[!t]
    \centering
    \includegraphics[width=0.75\linewidth]{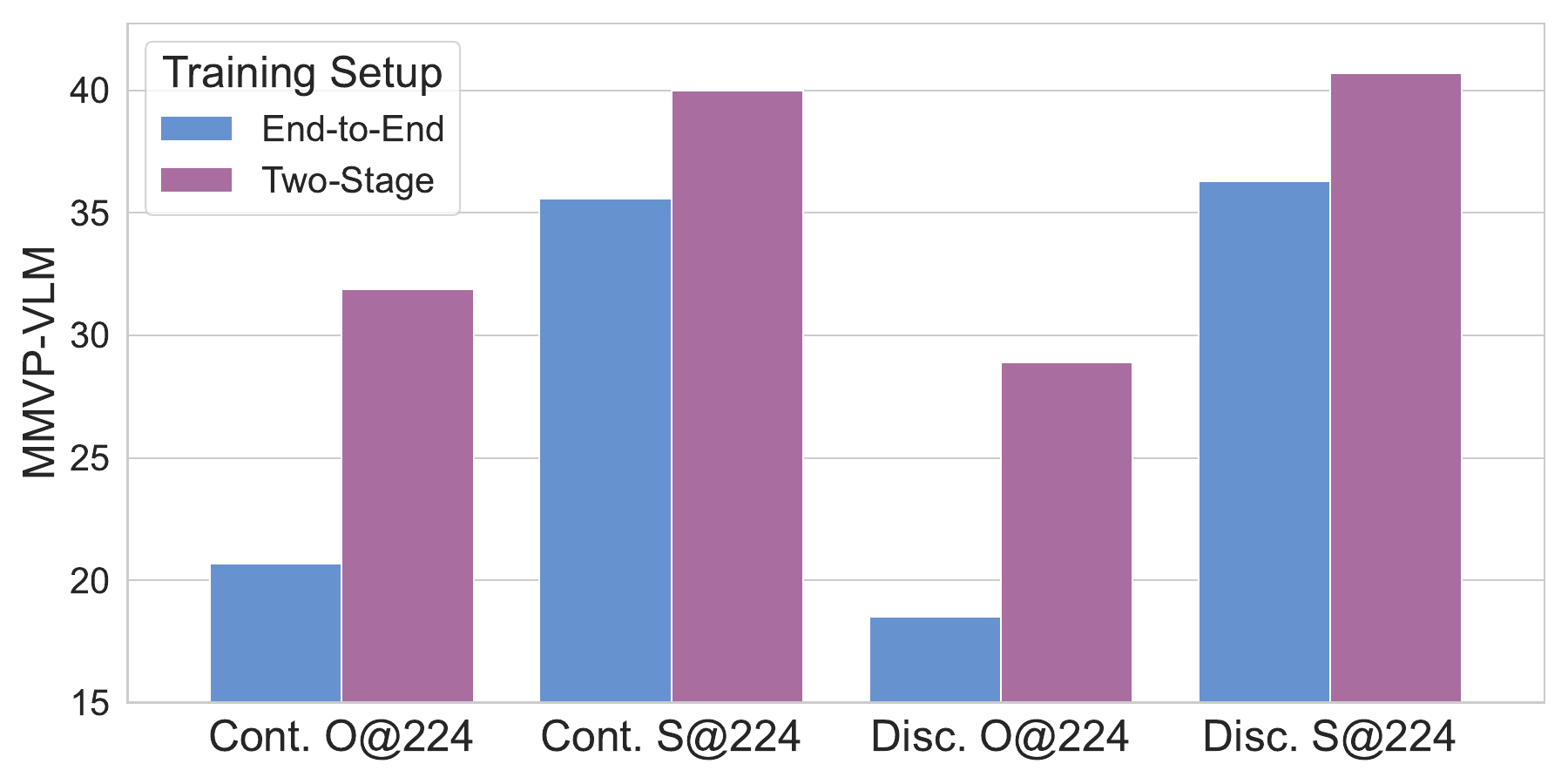}
    \vspace{-7pt}
    \caption{Comparison of CLIP with end-to-end and the proposed two-stage training on MMVP-VLM. Here, Cont. and Disc. denote continuous and discrete denoisers. O: OpenAICLIP. S: SigLIP.}
    \vspace{-7pt}
    \label{fig:training-stages}
\end{figure}

\vspace{-12pt}
\paragraph{Key Point \#2.2: Timestamp Sampling for Continuous Denoisers.}
Timestamp sampling of continuous denoisers is also pivotal for $\boldsymbol{v}_\theta$ to learn the fine-grained knowledge from $\boldsymbol{g}_\phi$, \ie, $I(V;G_1)$. We compare our proposed \emph{scaled} Logit-Normal sampling with standard uniform sampling, as shown in Table~\ref{tab:timestamp}.
Compared to uniform sampling, ours favors sampling closer to the middle ($t=0.5$), \ie, in the middle of two distributions $\boldsymbol{x}_t=t\boldsymbol{x}_1+(1-t)\boldsymbol{x}_0$, making denoising more challenging and more beneficial for enhancing $I(V;G_1)$. For example, our proposed distribution outperforms uniform sampling by 10.4\%, 7.4\% and 8.2\% on three CLIP backbones in Table~\ref{tab:timestamp}.
Additionally, when the scale $s$ is too small (\eg, $s=0.1$, sampling too around 0.5) or too large (\eg, $s=10$, sampling close to 0 or 1), the lack of diversity in $t$ can lead to suboptimal results due to the lack of diversity. In this work, we set $s=1$ by default.

\begin{table}[!t]
\setlength\tabcolsep{4pt}
\centering
\renewcommand{\arraystretch}{.7}
\caption{Comparison of timestamp sampling in continuous denoisers on MMVP-VLM. O: OpenAICLIP. M: MetaCLIP.}
\vspace{-10pt}
\label{tab:timestamp}
\resizebox{.75\linewidth}{!}{
\begin{tabular}{@{}lcccc@{}}
\toprule
Distribution & Scale & \multicolumn{1}{l}{O@224} & \multicolumn{1}{l}{O@336} & \multicolumn{1}{l}{M@224} \\ \midrule
Uniform & N/A & 21.5 & 22.2 & 23.7 \\ \midrule \midrule
\multirow{5}{*}{Logit-Normal} & 0.1 & 27.4 & 25.9 & 26.7 \\
 & 0.5 & 28.2 & 28.9 & 29.6 \\
 & 1.0 & \textbf{31.9} & \textbf{29.6} & \textbf{31.9} \\
 & 5.0 & 24.5 & 25.9 & 25.9 \\
 & 10.0 & 20.7 & 20.0 & 21.5 \\ \bottomrule
\end{tabular}
}
\vspace{-7pt}
\end{table}

\begin{table}[!t]
\setlength\tabcolsep{3pt}
\centering
\renewcommand{\arraystretch}{1.0}
\caption{Performance on SigLIP@224 across different sizes of lightweight \colorbox{color2}{continuous} and \colorbox{color3}{discrete} denoisers.}
\vspace{-10pt}
\label{tab:denoiser-size}
\resizebox{.9\linewidth}{!}{
\begin{tabular}{@{}cccccc@{}}
\toprule
\multirow{2}{*}{Continuous} & \cellcolor{color2}\#DiT Blocks (MM+Single) & \cellcolor{color2}1+2 & \cellcolor{color2}2+4 & \cellcolor{color2}3+6 & \cellcolor{color2}4+8 \\ \cmidrule(l){2-6} 
 & \cellcolor{color2}MMVP-VLM & \cellcolor{color2}41.5 & \cellcolor{color2}\textbf{42.2} & \cellcolor{color2}\textbf{42.2} & \cellcolor{color2}41.5 \\ \midrule \midrule
\multirow{2}{*}{Discrete} & \cellcolor{color3}\#Perceiver Layers & \cellcolor{color3}2 & \cellcolor{color3}4 & \cellcolor{color3}6 & \cellcolor{color3}8 \\ \cmidrule(l){2-6} 
 & \cellcolor{color3}MMVP-VLM & \cellcolor{color3}41.5 & \cellcolor{color3}43.7 & \cellcolor{color3}\textbf{45.2} & \cellcolor{color3}43.7 \\ \bottomrule
\end{tabular}
}
\vspace{-7pt}
\end{table}

\vspace{-12pt}
\paragraph{Key Point \#2.3: Sizes of lightweight denoisers.}
We further explore the impact of the size of lightweight denoisers. For the continuous RF, we consider the number of blocks in MM-DiT and Single DiT. We consider the number of layers for Perceiver. Table~\ref{tab:denoiser-size} demonstrates that the denoiser could perform remarkably well with a relatively small size, indicating the efficiency of our lightweight denoisers.

\vspace{-12pt}
\paragraph{Key Point \#3: Continuous and Discrete Denoisers.}
Table~\ref{tab:mmvp-continuous-discrete} demonstrates the performance with continuous and discrete denoisers. Both of them surpass previous work~\cite{wang2025diffusion} on various backbones. For example, the discrete denoiser obtains a 4.5\% performance gain on SigLIP@224~\cite{zhai2023sigmoid}.
In summary, our method is general and applies to both continuous and discrete models. It is efficient with lightweight denoisers but strong enough to outperform prior arts~\cite{wang2025diffusion}.
Notably, previous Key Points \#1$\sim$\#2 are consistently applicable to both continuous and discrete denoisers, further highlighting the versatility of our method.

\subsection{Further Analysis}

\paragraph{Why are improvements on SigLIP relatively small?}
In Table~\ref{tab:main-mmvp-clip}, we observe that the improvement on SigLIP is relatively smaller compared to OpenAICLIP and MetaCLIP. Specifically, the performance gain over the original SigLIP is $\sim 3.7\%$, less than that for others, \ie, $>10\%$.
Unlike the other two backbones~\cite{radford2021learning,xu2024demystifying}, SigLIP~\cite{zhai2023sigmoid} does not explicitly train a distinct class token. In practice, we extract the \texttt{pooler\_output} of SigLIP as the condition for the denoiser, which is obtained by aggregating all local tokens through attention and linear layers.
We attribute the relatively small improvement on SigLIP to the indirect leakage of local information through the \texttt{pooler\_output}, which hinders the enhancement of $I(V; G_1)$. This is consistent with the discussion in Sec.~\ref{subsec:condition} and the results in Fig.~\ref{fig:ratio-local}.

\begin{table}[!t]
\setlength\tabcolsep{4pt}
\centering
\renewcommand{\arraystretch}{1}
\caption{Performance of our method with our \colorbox{color2}{continuous} and \colorbox{color3}{discrete} denoisers on MMVP-VLM (average of all visual patterns). \textbf{Bold} and \underline{underline} indicate the best and the second best.}
\vspace{-10pt}
\label{tab:mmvp-continuous-discrete}
\resizebox{.8\linewidth}{!}{
\begin{tabular}{@{}lccc@{}}
\toprule
Method & OpenAI@224 & SigLIP@224 & SigLIP@384 \\ \midrule
DIVA & 25.9 & 40.7 & 38.5 \\
\cellcolor{color2}Continuous & \cellcolor{color2}\textbf{31.9} & \cellcolor{color2}\underline{42.2} & \cellcolor{color2}\underline{40.0} \\
\cellcolor{color3}Discrete & \cellcolor{color3}\underline{28.9} & \cellcolor{color3}\textbf{45.2} & \cellcolor{color3}\textbf{40.7} \\ \bottomrule
\end{tabular}
}
\vspace{-7pt}
\end{table}

\begin{table}[!t]
\setlength\tabcolsep{3pt}
\centering
\renewcommand{\arraystretch}{1.0}
\caption{Efficiency comparison of our lightweight RF denoiser with pre-trained \texttt{FLUX.1-dev}.}
\vspace{-10pt}
\label{tab:efficiency}
\resizebox{.95\linewidth}{!}{
\begin{tabular}{@{}lccccc@{}}
\toprule
\multirow{2}{*}{Denoiser} & \multicolumn{3}{c}{Efficiency} & \multicolumn{2}{c}{MMVP-VLM} \\ \cmidrule(l){2-4}  \cmidrule(l){5-6} 
 & \#Params & Memory & Time/100 iters & OpenAI & Meta-H \\ \midrule
Pre-trained & 11.90B & 37.33G & 198.57s & \textbf{32.6} & \textbf{37.1} \\
\RC{30}Lightweight & \textbf{1.31B} & \textbf{13.07G} & \textbf{20.55s} & 31.9 & \textbf{37.1} \\ \bottomrule
\end{tabular}
}
\vspace{-7pt}
\end{table}

\vspace{-12pt}
\paragraph{Efficiency analysis compared with pre-trained \texttt{FLUX}.}
We provide a comparison between our lightweight RF ($n=2$) and original \texttt{FLUX.1-dev}~\cite{flux2024} across the following dimensions: \#params of denoisers, per-device GPU memory and training time of 100 iterations. To ensure fair comparisons, we fix a per-device batch size of 2. As Table~\ref{tab:efficiency} shows, our lightweight denoiser is much more efficient than the pre-trained heavy one. Specifically, our lightweight denoiser has approximately 1/10 of the parameters, occupies about 1/3 of the memory, and is 10 times faster in training, while the final performance remains comparable.

%% file: sec/6_conclusive_remarks.tex
\section{Conclusive Remarks}
\label{sec:conclusive-remarks}

In this paper, we delve into the underlying principles of how generative models enhance visual representations. We innovatively uncover that the perfect generation does not always yield optimal representations. The pivot is to learn useful knowledge from the generative model while mitigating irrelevant information.
Our key findings lie in three aspects. (1) Conditioning mechanism. We found that local tokens could make the reconstruction task too easy, while class token \emph{alone} as the condition makes the reconstruction task meaningful and significantly enhances visual representations. (2) Denoising configurations. We propose a novel two-stage post-training method to enable vision encoders committed to learning fine-grained knowledge while alleviating irrelevant content. (3) Our model design enables both continuous and discrete denoisers to effectively enhance visual representations.
Vision-centric evaluations demonstrate that our method with lightweight denoisers can significantly outperform previous methods relying on heavy pre-trained generative models.
We hope this work will inspire further in-depth explorations into the synergy between generative and discriminative models, as well as the relationship between generation and understanding tasks.

%% file: sec/X_appendix.tex
\appendix
\maketitlesupplementary


\section*{Overview}

In this appendix, we provide additional descriptions of the following contents:
\begin{itemize}
    \item Relationship with prior works in Appendix~\ref{sec:appendix-relation}, including some discussions about the differences.
    \item More training details of hyperparameters in Appendix~\ref{sec:appendix-training-details}.
    \item Diagrams of various timestamp sampling distributions in Appendix~\ref{sec:appendix-timestamp}.
    \item Additional experimental results in Appendix~\ref{sec:appendix-experiments}.
    \item Additional qualitative results and cases of the enhanced CLIP (Appendix~\ref{sec:appendix-qualitative-clip}) and MLLMs with our enhanced CLIP (Appendix~\ref{sec:appendix-qualitative-mllm}).
    \item We also attach algorithms of our two-stage training with continuous and dicrete denoisers in Appendix~\ref{sec:appendix-algorithms}.
\end{itemize}

\section{Relationship with Prior Works}
\label{sec:appendix-relation}

In this paper, we propose a two-stage post-training method to enhance discriminative models' fine-grained visual representations. For discriminative models, we primarily choose CLIP~\cite{radford2021learning}, considering its wide range of applications. Specifically, CLIP is inherently a vision-language model, capable of image-text retrieval and matching. Additionally, CLIP ViT is widely employed as a visual encoder in Multimodal Large Language Models (MLLMs). Note that our approach follows a post-training paradigm, where we enhance the fine-grained capabilities of a pre-trained CLIP ViT, while preserving its original global semantics.

\vspace{-12pt}
\paragraph{Comparison with DIVA~\cite{wang2025diffusion}.}
DIVA is a pioneering work and proposes to enhance visual representations of CLIP ViT through diffusion feedback. It independently enhances CLIP ViT's visual representations with the guidance of pre-trained stable diffusion~\cite{rombach2022high}. Similar to DIVA, our work focuses on enhancing CLIP ViT's internal visual representations. The enhanced CLIP itself could be a more competent vision-language model with better image-text retrieval performance. Furthermore, the enhanced CLIP ViT serves as a \emph{plug-and-play} module and could be seamlessly plugged into MLLMs. When using the same training recipes but with the enhanced vision encoder, MLLMs could be more capable on several vision-centric benchmarks, with better fine-grained perception of visual details and overcoming visual shortcomings brought about by the original CLIP.

Different from DIVA, we delve into the underlying principles of how generative models enhance vision models from various orthogonal dimensions. Notably, we only employ lightweight denoisers without pre-trained weights of heavy generative models. Our method is efficient yet stronger than DIVA. We also provide several key insights about how to enhance visual representations, \ie, conditioning mechanisms and training configurations. We further explore the implementation of both continuous and discrete generative models. When equipped with corresponding tailor-made designs, both continuous and discrete denoisers outperform DIVA.

\vspace{-12pt}
\paragraph{Comparison with ROSS~\cite{wang2025reconstructive}.}
Ross is a pioneering work that explores the intrinsic signals in vision modality and proposes to append vision-centric self-supervision into the training of MLLMs. The core difference between ROSS and our method is that, ROSS is directly oriented to training better MLLMs. In most cases, ROSS freezes CLIP ViT and enhances the vision-centric performance of MLLMs through the parameters of LLMs. In contrast, our method is directly oriented to enhance CLIP ViT's visual representations. Our method is more general, and the resulting enhanced CLIP could be plugged into various MLLMs. In summary, we independently enhance CLIP ViT, which could be merged into MLLMs for further enhancements, while ROSS directly enhances MLLMs with the ViT frozen.

\section{More Training Details}
\label{sec:appendix-training-details}

\vspace{-5pt}
\paragraph{Default training settings.}
Our training process consists of two stages, each involving one epoch on the CC3M~\cite{sharma2018conceptual} dataset. We choose AdamW as the optimizer, with a learning rate of 1e-4 and 1e-5 for Stage-1 and Stage-2, respectively. At Stage-2, we optimize the visual encoder using LoRA with a rank of 16. We train the model on 8 GPUs with a per-device batch size of 16, and the gradient accumulation steps are set as 2, resulting in a global batch size of 256. We plug LoRA to CLIP ViT, with a rank of 16, and an $\alpha$ of 16. Additionally, we employ dropout with a ratio of 0.1 within LoRA.

\vspace{-10pt}
\paragraph{Detailed settings in Fig.~1 of main manuscript.}
The default settings are: a lightweight denoiser with 2 MM-DiT and 4 Single-DiT blocks, using only the \texttt{[CLS]} as the condition, and two-stage training with 100,000 steps in stage-1 and 5,000 steps in stage-2. Each of the four aspects in Fig.~1(b) modifies \emph{only} one parameter or dimension at a time. Specifically, (i) changes \#iters in stage-2 from 100 to 10,000. (ii) varies the number of denoiser-blocks ($n\times$MM-DiT+$2n\times$Single-DiT) from $n=1$ to $n=3$. (iii) conditions denoisers with \texttt{[CLS]} along with $n\%$ of local tokens. (iv) compares the lightweight denoiser with $n=4$ and the pretrained heavy $\texttt{FLUX}$ with $n=19$.

\begin{figure*}[!t]
    \centering
    \includegraphics[width=.7\linewidth]{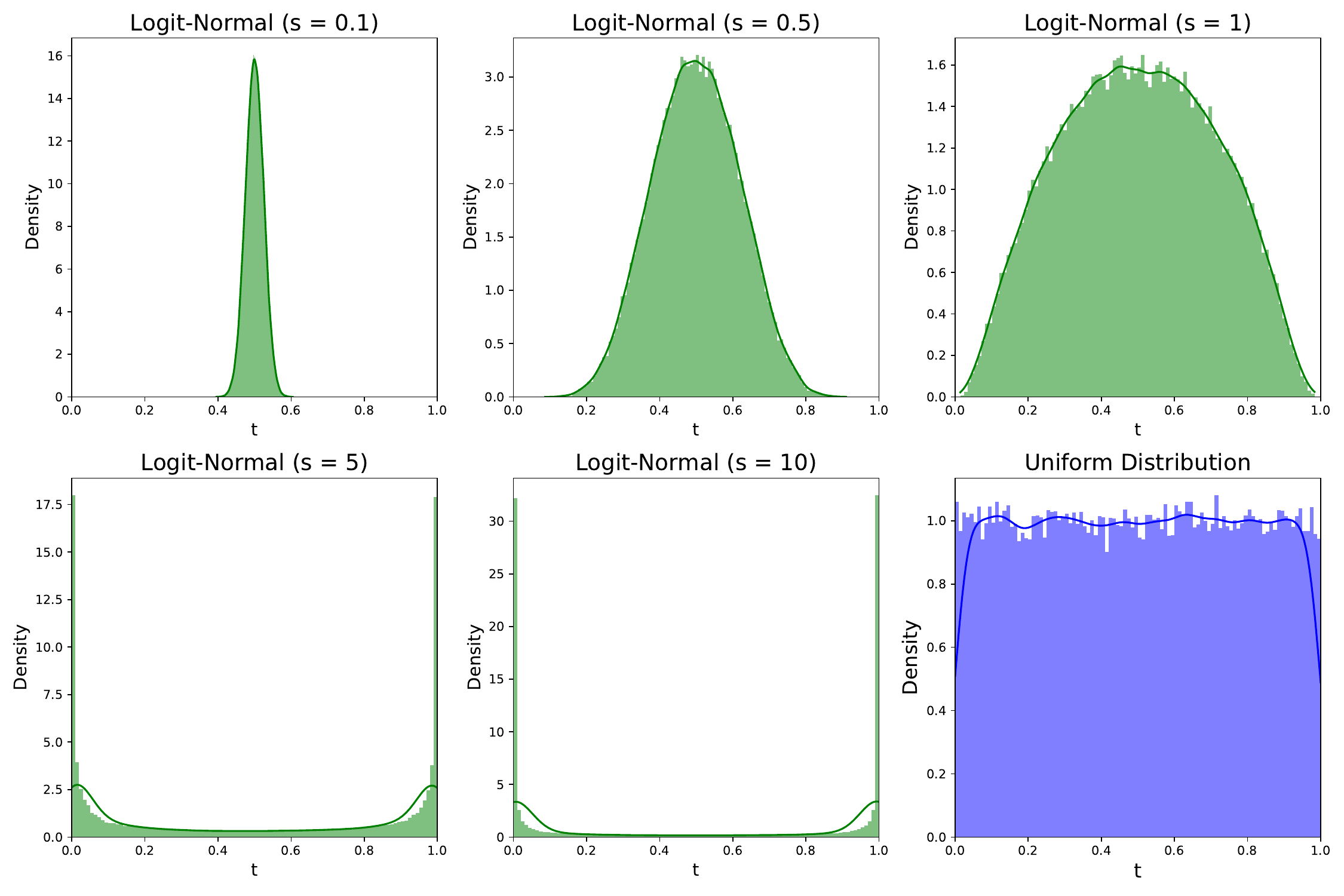}
    \vspace{-7pt}
    \caption{Probability density function of different distributions.}
    \vspace{-7pt}
    \label{fig:appendix-t-distribution}
\end{figure*}

\vspace{-10pt}
\paragraph{Detailed settings in Table~7 of main manuscript.}
The pretrained $\texttt{FLUX}$ is also trained under the same setting, \ie, two-stage training and \emph{only} the \texttt{[CLS]} serves as the condition. Without these proposed keypoints, pretrained denoiser also fails to gain desirable results, \eg, 32.9$\to$22.2 on MMVP, further indicating the generality of our method.

\section{Diagrams of Timestamp Sampling}
\label{sec:appendix-timestamp}

The \emph{scaled} Logit-Normal timestamp sampling is as follows:
\begin{equation}
    t=\texttt{sigmoid}(s\cdot\varepsilon),\quad \text{where}\ \ \varepsilon\sim\mathcal{N}(0,1).
    \label{eq:appendix-t-sampling}
\end{equation}
We provide some illustrative diagrams to show the distribution of several candidate distributions, as shown in Fig.~\ref{fig:appendix-t-distribution}. In our \emph{scaled} Logit-Normal sampling, as $s$ decreases, the distribution becomes more focused on sampling around the middle ($t=0.5$). Conversely, as $s$ increases, the distribution becomes more biased towards sampling at the extremes, \ie, $t=0$ or 1.

\section{More Experimental Results}
\label{sec:appendix-experiments}

\begin{figure}[!t]
    \centering
    \includegraphics[width=.9\linewidth]{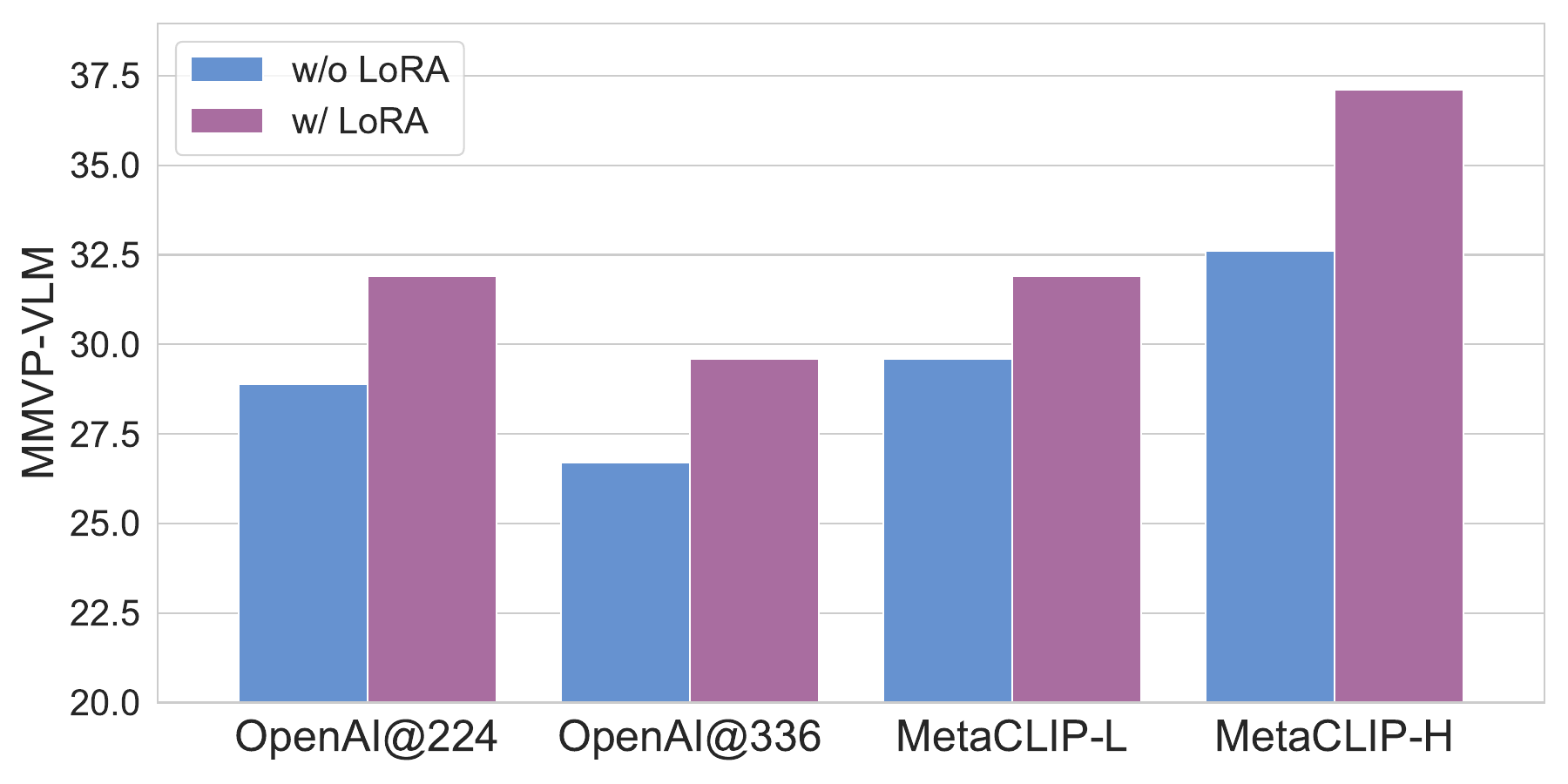}
    \vspace{-7pt}
    \caption{The effect of LoRA on several CLIP backbones.}
    \vspace{-7pt}
    \label{fig:appendix-lora}
\end{figure}

\paragraph{The effect of LoRA.}
In Stage-2, we apply LoRA to the visual model. The reason is that direct training on the visual model causes rapid updates, which can easily damage the model’s high-level semantics and lead to overfitting. By using LoRA, the model can be trained on a larger variety of samples, allowing it to learn more generalizable and fine-grained representations. We conduct experiments on several CLIP backbones, and compare the performance with direct training and LoRA training, as shown in Fig.~\ref{fig:appendix-lora}. The performance with LoRA for visual encoder consistently outperforms the cases of direct training.

\vspace{-12pt}
\paragraph{Whether to update the denoiser and projector in Stage-2.}
In the main text, we argue that in Stage-1, the visual encoder should be fixed, and we train the denoiser and projector. In this way, the projector could learn to bridge the gap between the feature spaces, which serves as the irrelevant information $G_2$ for visual enhancements. In Stage-2, we begin to train CLIP ViT to enhance its visual representations. We empirically found that whether the denoiser and projector are updated in Stage-2 has marginal impacts on the final results, as long as stage-1 training is sufficient. The results are shown in Fig.~\ref{fig:appendix-tune-only-all}.

\begin{figure*}[!t]
    \centering
    \includegraphics[width=.95\linewidth]{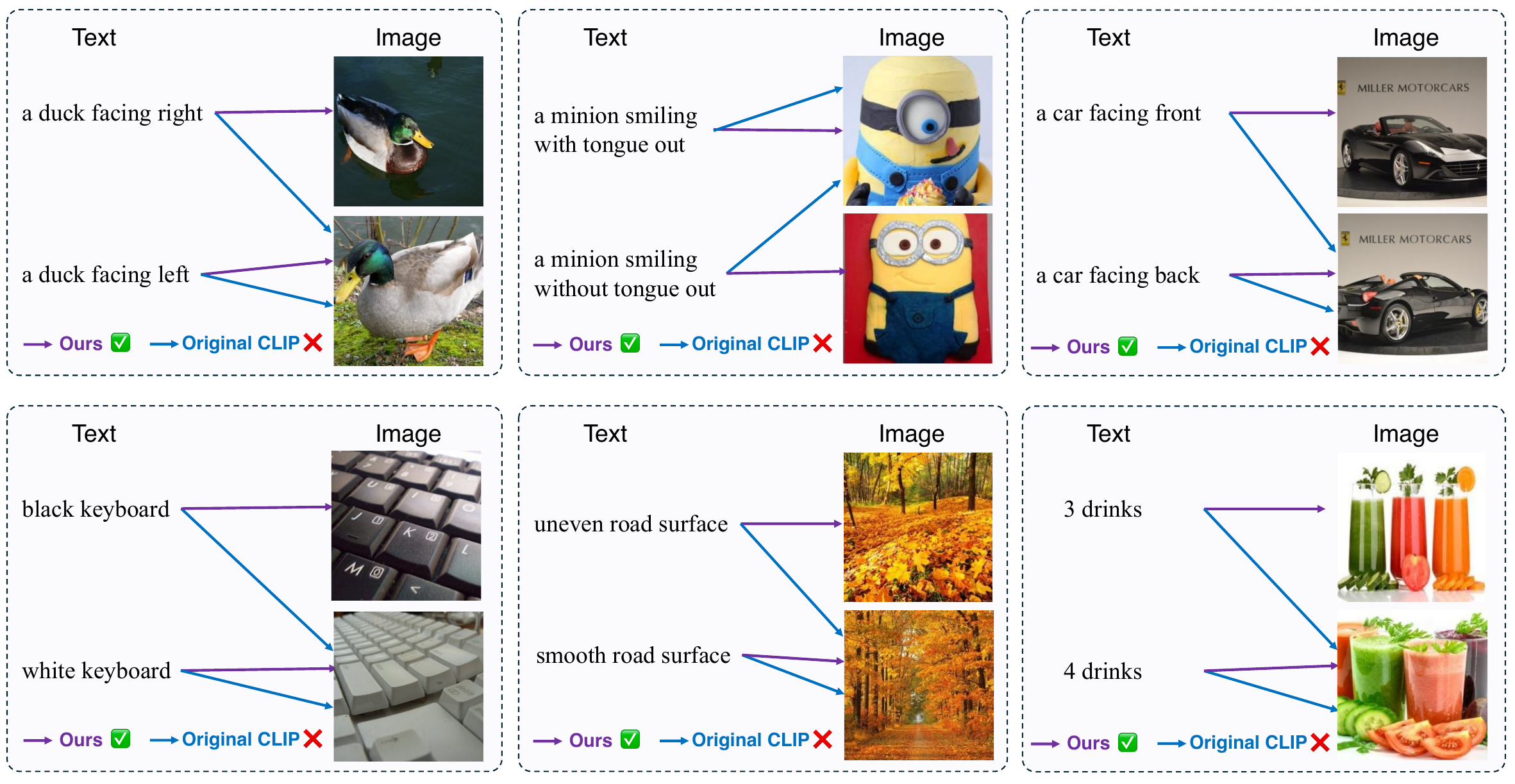}
    \vspace{-7pt}
    \caption{Qualitative results of CLIP. The enhanced CLIP overcomes the original visual shortcomings in fine-grained details.}
    \vspace{-7pt}
    \label{fig:appendix-qualitative-clip}
\end{figure*}

\begin{figure}[!t]
    \centering
    \includegraphics[width=.85\linewidth]{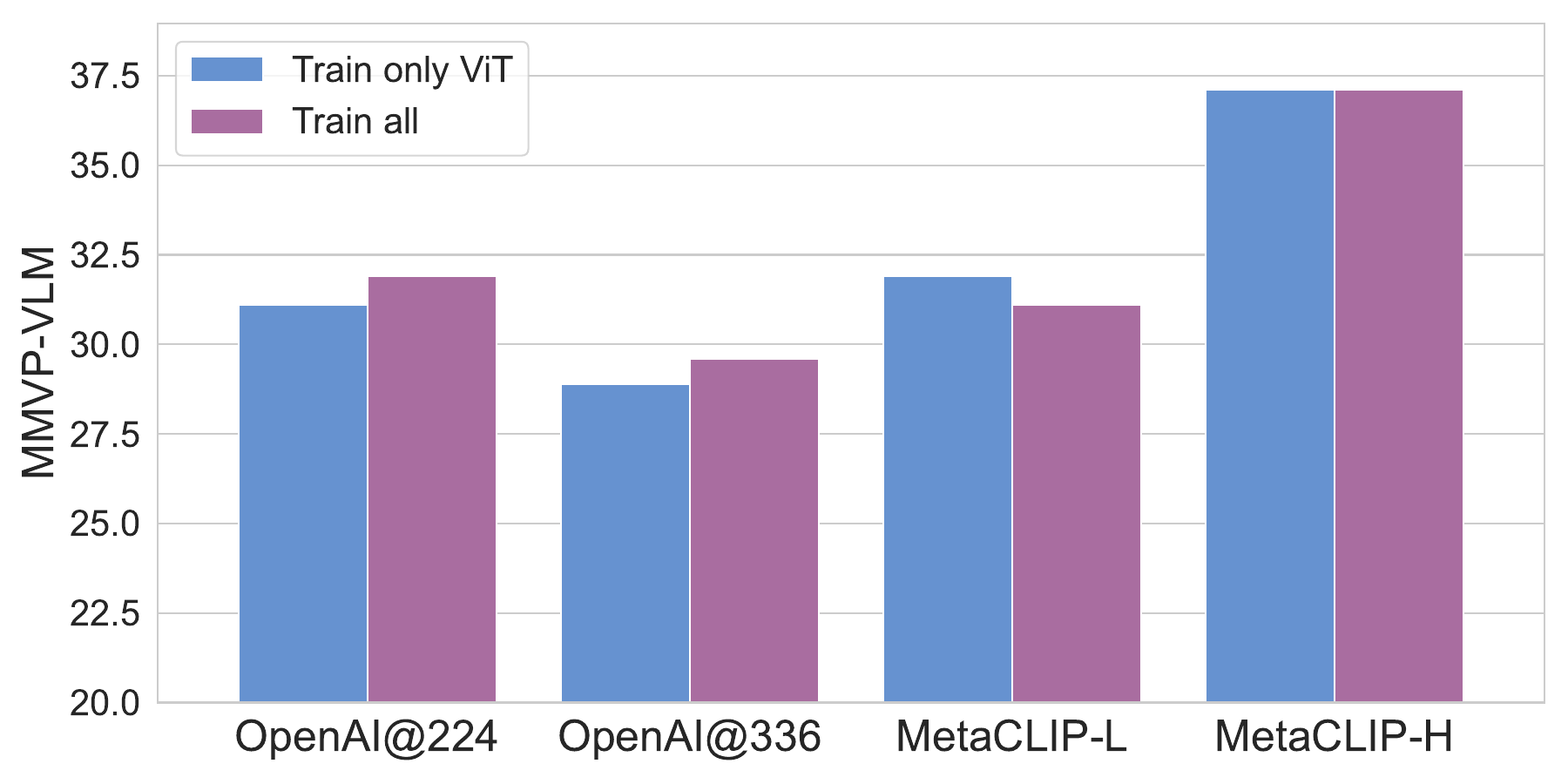}
    \vspace{-7pt}
    \caption{The performance of whether to update the denoiser and the projector in Stage-2.}
    \vspace{-7pt}
    \label{fig:appendix-tune-only-all}
\end{figure}

\begin{table}[!t]
\setlength\tabcolsep{3pt}
\centering
\renewcommand{\arraystretch}{1}
\caption{Performance of various mask ratios on OpenAICLIP@224.}
\vspace{-7pt}
\label{tab:appendix-mask-ratio}
\resizebox{.95\linewidth}{!}{
\begin{tabular}{@{}lcccccccc@{}}
\toprule
Mask Ratio (\%) & 50 & 60 & 70 & 75 & 80 & 85 & 90 & random (50-90) \\ \midrule
MMVP-VLM & 28.1 & 27.4 & \textbf{28.9} & 27.4 & 26.7 & 25.9 & 25.9 & \textbf{28.9} \\ \bottomrule
\end{tabular}
}
\vspace{-7pt}
\end{table}

\vspace{-12pt}
\paragraph{Performance with various mask ratios.}
In the discrete denoiser, we apply masking mechanisms. Here, we provide experimental results across various mask ratios of OpenAICLIP@224, as shown in Table~\ref{tab:appendix-mask-ratio}.

\section{Qualitative Results of CLIPs}
\label{sec:appendix-qualitative-clip}

We provide further qualitative results of the original CLIP and our enhanced CLIP, as shown in Fig.~\ref{fig:appendix-qualitative-clip}. The enhanced CLIP overcomes original visual shortcomings in fine-grained details, including color, quantity, structural characteristics and state.

\begin{figure*}[!t]
    \centering
    \includegraphics[width=0.95\linewidth]{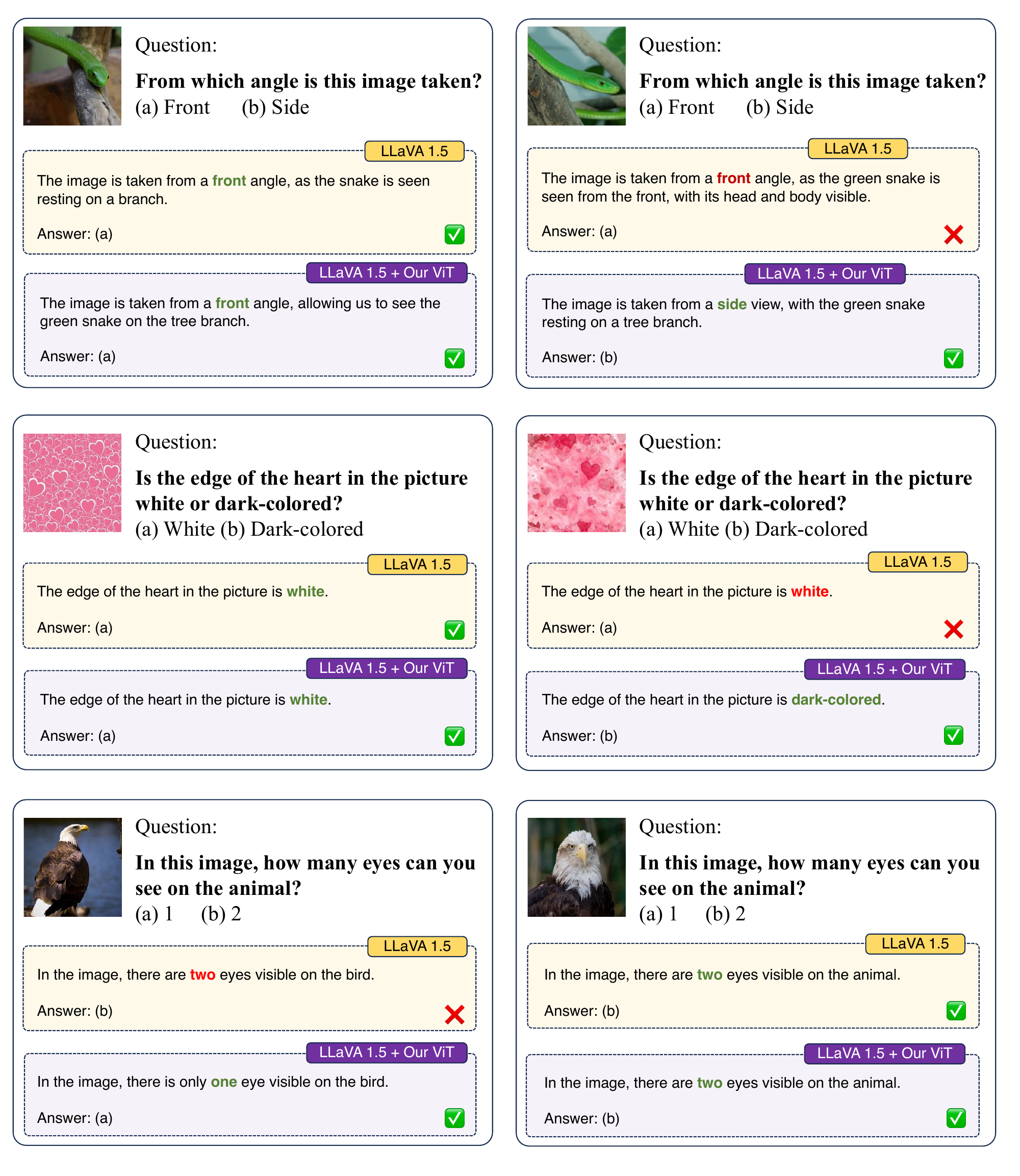}
    \caption{Qualitative results of MLLMs on MMVP-MLLM benchmark. When equipped with our enhanced CLIP, MLLMs produce better vision-centric performance.}
    \label{fig:appendix-qualitative-mllm}
\end{figure*}

\section{Qualitative Results of MLLMs}
\label{sec:appendix-qualitative-mllm}

We provide qualitative results of LLaVA-1.5 with original CLIP ViT and our enhanced CLIP ViT, as shown in Fig~\ref{fig:appendix-qualitative-mllm}. Our enhanced visual model could further boost MLLMs' fine-grained visual perception abilities.

\section{Algorithms}
\label{sec:appendix-algorithms}

For a clearer and more thorough understanding of our method, we attach the algorithm details of two-stage post-training with continuous and discrete denoisers in Algorithm~\ref{alg:continuous-denoiser} and Algorithm~\ref{alg:discrete-denoiser}, respectively.

\begin{algorithm*}[p]
    \small
    \caption{Two-stage Visual Enhancements with Continuous Lightweight Denoiser}
    \label{alg:continuous-denoiser}
    \begin{algorithmic}[1]
        \Require {Lightweight and random-initialized denoiser $\boldsymbol{g}_\phi(\cdot)$, with lightweight \texttt{FLUX}-like architecture (MM-DiT + Single-DiT).}
        \Require {Pre-trained CLIP ViT $\boldsymbol{v}_\theta(\cdot)$ for fine-grained visual representation enhancements.}
        \Require {Random initialized projector $\boldsymbol{h}_\omega(\cdot)$ to bridge the feature space of $\boldsymbol{v}_\theta$ and condition space of $\boldsymbol{g}_\phi$.}
        \Require {The scale hyperparameter $s$ in the proposed \emph{scaled} Logit-Normal sampling.}
        \Require {Pre-trained VAE $\texttt{vae}(\cdot)$ to provide latent space for generative modeling.}
        \Require {Image-only training dataset $\mathcal{D}$ without annotations.}
        \State {\textcolor{iccvblue}{\texttt{\# =================================== Stage-1 ==================================}}}
        \For {$\boldsymbol{x}$ in $\mathcal{D}$}
            \State $\triangleright$ Prepare input data for generative modeling in latent space: \quad $\widetilde{\boldsymbol{x}_1}=\texttt{vae}(\boldsymbol{x})$ and $\widetilde{\boldsymbol{x}_0}\sim \mathcal{N}(\boldsymbol{0},\boldsymbol{I})$.
            \State $\triangleright$ Interpolating in the feature space: \quad
            $\widetilde{\boldsymbol{x}_t}=t\widetilde{\boldsymbol{x}_1}+(1-t)\widetilde{\boldsymbol{x}_0}$.
            \State $\triangleright$ Visual encoding as conditions for denoisers: \quad $\boldsymbol{h}_\omega\circ\boldsymbol{v}_\theta(\boldsymbol{x})$.
            \State $\triangleright$ Timestamp sampling via \emph{scaled} Logit-Normal distributions: \quad
            $\varepsilon\sim\mathcal{N}(0,1)$ then $t=\texttt{sigmoid}(s\cdot\varepsilon)$.
            \State $\triangleright$ Denoising regression objective (flow matching): \quad\quad\texttt{\textcolor{red}{\# only update $\boldsymbol{g}_\phi$ and $\boldsymbol{h}_\omega$.}} 
            \begin{equation*}
                \arg\min_{\phi,\omega}\mathbb{E}_{t,\boldsymbol{x},\widetilde{\boldsymbol{x}}_0,\widetilde{\boldsymbol{x}}_1}\big\Vert (\widetilde{\boldsymbol{x}_1}-\widetilde{\boldsymbol{x}_0})-\boldsymbol{g}_\phi\big(\widetilde{\boldsymbol{x}_t},t,\boldsymbol{h}_\omega\circ\boldsymbol{v}_\theta(\boldsymbol{x})\big) \big\Vert_2^2.
            \end{equation*}
        \EndFor
        \vspace{5pt}
        \State {\textcolor{iccvblue}{\texttt{\# =================================== Stage-2 ==================================}}}
        \State Plug LoRA upon $\boldsymbol{v}_\theta$.
        \For {$\boldsymbol{x}$ in $\mathcal{D}$}
            \State $\triangleright$ Prepare input data for generative modeling in latent space: \quad $\widetilde{\boldsymbol{x}_1}=\texttt{vae}(\boldsymbol{x})$ and $\widetilde{\boldsymbol{x}_0}\sim \mathcal{N}(\boldsymbol{0},\boldsymbol{I})$.
            \State $\triangleright$ Interpolating in the feature space: \quad
            $\widetilde{\boldsymbol{x}_t}=t\widetilde{\boldsymbol{x}_1}+(1-t)\widetilde{\boldsymbol{x}_0}$.
            \State $\triangleright$ Visual encoding as conditions for denoisers: \quad $\boldsymbol{h}_\omega\circ\boldsymbol{v}_\theta(\boldsymbol{x})$.
            \State $\triangleright$ Timestamp sampling via \emph{scaled} Logit-Normal distributions: \quad
            $\varepsilon\sim\mathcal{N}(0,1)$ then $t=\texttt{sigmoid}(s\cdot\varepsilon)$.
            \State $\triangleright$ Denoising regression objective (flow matching): \quad\quad\texttt{\textcolor{red}{\# update $\boldsymbol{v}_\theta$. Optional: $\boldsymbol{g}_\phi$ and $\boldsymbol{h}_\omega$.}}
            \begin{equation*}
                \arg\min_{\theta}\mathbb{E}_{t,\boldsymbol{x},\widetilde{\boldsymbol{x}}_0,\widetilde{\boldsymbol{x}}_1}\big\Vert (\widetilde{\boldsymbol{x}_1}-\widetilde{\boldsymbol{x}_0})-\boldsymbol{g}_\phi\big(\widetilde{\boldsymbol{x}_t},t,\boldsymbol{h}_\omega\circ\boldsymbol{v}_\theta(\boldsymbol{x})\big) \big\Vert_2^2.
            \end{equation*}
        \EndFor
        \Ensure {The enhanced visual model $\boldsymbol{v}_\theta^\star$ with stronger fine-grained representations.}
    \end{algorithmic}
\end{algorithm*}

\begin{algorithm*}[p]
    \small
    \caption{Two-stage Visual Enhancements with Discrete Lightweight Denoiser}
    \label{alg:discrete-denoiser}
    \begin{algorithmic}[1]
        \Require {Lightweight and random-initialized denoiser $\boldsymbol{g}_\phi(\cdot)$, instantiated with a lightweight Perceiver.}
        \Require {Pre-trained CLIP ViT $\boldsymbol{v}_\theta(\cdot)$ for fine-grained visual representation enhancements.}
        \Require {Random initialized projector $\boldsymbol{h}_\omega(\cdot)$ to bridge the feature space of $\boldsymbol{v}_\theta$ and condition space of $\boldsymbol{g}_\phi$.}
        \Require {Mask ratio $r$ for discrete modeling.}
        \Require {Pre-trained VQ-GAN $\texttt{vq-gan}(\cdot)$ to discrete indices for generative modeling.}
        \Require {Image-only training dataset $\mathcal{D}$ without annotations.}
        \State {\textcolor{iccvblue}{\texttt{\# =================================== Stage-1 ==================================}}}
        \For {$\boldsymbol{x}$ in $\mathcal{D}$}
             \State $\triangleright$ Obtain latent embeddings and corresponding discrete indices of input data in VQ-GAN's codebook: \quad
             $\widetilde{\boldsymbol{x}}, s=\texttt{vq-gan}(\boldsymbol{x})$.
            \State $\triangleright$ Masking $\boldsymbol{x}$'s tokens with ratio $r$ to obtain masked part $\widetilde{\boldsymbol{x}}_{mask}, s_{mask}$ and unmasked part $\widetilde{\boldsymbol{x}}_{unmask}, s_{unmask}$.
            \State $\triangleright$ Visual encoding and obtain conditions via cross-attention for denoisers:
            \begin{align*}
                Q&=\widetilde{\boldsymbol{x}}_{unmask}, \\
                K,V&=\texttt{concat}\big(\widetilde{\boldsymbol{x}}_{unmask};\boldsymbol{h}_\omega\circ\boldsymbol{v}_\theta(\boldsymbol{x})\big), \\
                \boldsymbol{c}_{\omega,\theta}&=\texttt{cross-attn}(Q,K,V).
            \end{align*}
            \State $\triangleright$ Denoising cross-entropy objective (masked index prediction): \quad\quad\texttt{\textcolor{red}{\# only update $\boldsymbol{g}_\phi$ and $\boldsymbol{h}_\omega$.}} 
            \begin{equation*}
                \arg\min_{\phi,\omega}\mathbb{E}_{\boldsymbol{x}}-\log \prod_{i=1}^L \boldsymbol{g}_\phi\big(s_{mask}|s_{unmask},\boldsymbol{c}_{\omega,\theta}\big).
            \end{equation*}
        \EndFor
        \vspace{5pt}
        \State {\textcolor{iccvblue}{\texttt{\# =================================== Stage-2 ==================================}}}
        \State Plug LoRA upon $\boldsymbol{v}_\theta$.
        \For {$\boldsymbol{x}$ in $\mathcal{D}$}
            \State $\triangleright$ Obtain latent embeddings and corresponding discrete indices of input data in VQ-GAN's codebook: \quad
             $\widetilde{\boldsymbol{x}}, s=\texttt{vq-gan}(\boldsymbol{x})$.
            \State $\triangleright$ Masking $\boldsymbol{x}$'s tokens with ratio $r$ to obtain masked part $\widetilde{\boldsymbol{x}}_{mask}, s_{mask}$ and unmasked part $\widetilde{\boldsymbol{x}}_{unmask}, s_{unmask}$.
            \State $\triangleright$ Visual encoding and obtain conditions via cross-attention for denoisers:
            \begin{align*}
                Q&=\widetilde{\boldsymbol{x}}_{unmask}, \\
                K,V&=\texttt{concat}\big(\widetilde{\boldsymbol{x}}_{unmask};\boldsymbol{h}_\omega\circ\boldsymbol{v}_\theta(\boldsymbol{x})\big), \\
                \boldsymbol{c}_{\omega,\theta}&=\texttt{cross-attn}(Q,K,V).
            \end{align*}
            \State $\triangleright$ Denoising cross-entropy objective (masked index prediction): \quad\quad\texttt{\textcolor{red}{\# update $\boldsymbol{v}_\theta$. Optional: $\boldsymbol{g}_\phi$ and $\boldsymbol{h}_\omega$.}} 
            \begin{equation*}
                \arg\min_{\theta}\mathbb{E}_{\boldsymbol{x}}-\log \prod_{i=1}^L \boldsymbol{g}_\phi\big(s_{mask}|s_{unmask},\boldsymbol{c}_{\omega,\theta}\big).
            \end{equation*}
        \EndFor
        \Ensure {The enhanced visual model $\boldsymbol{v}_\theta^\star$ with stronger fine-grained representations.}
    \end{algorithmic}
\end{algorithm*}